\documentclass{article}
\pdfoutput=1    
\date{}               
\let\DRAFT=1 

\usepackage{amsmath}

\usepackage{amsfonts, amssymb, amsthm}

\usepackage{mathrsfs}
\def\ifundefined#1{\expandafter\ifx\csname#1\endcsname\relax}
\def\XBREAK{\ifundefined{DRAFT}~\else\break\fi}
\def\ZBREAK{\ifundefined{DRAFT}\break\else\fi}

\theoremstyle{plain}

\numberwithin{equation}{section}
\newtheorem{LEMMA}{Lemma}[section]
\newtheorem{THEOREM}{Theorem}[section]
\newtheorem{PROBLEM}{Problem}[section]
\newtheorem{COROLLARY}{Corollary}[THEOREM]

\theoremstyle{definition}

\newtheorem{DEFINITION}{Definition}[section]
\newtheorem{HYPOTHESIS}{Hypothesis}[section]

\theoremstyle{remark}

\newtheorem{OBSERVATION}{Observation}[section]

\def\LTWO{L_{2}}

\newcommand{\NORMLTWO}[1]{\ensuremath{ \left\| #1 \right\|_{\LTWO} } }

\newcommand{\NORMLTWOSQ}[1]{\ensuremath{ \left\| #1 \right\|^{2}_{\LTWO} } }  

\newcommand{\PRODLTWO}[1]{\ensuremath{ \left< #1 \right>_{\LTWO} } }

\newcommand{\NORMFROB}[1]{\ensuremath{ \left\| #1 \right\|_{F} } }

\newcommand{\NORMOP}[1]{\ensuremath{ \left\| #1 \right\|_{op} } }

\newcommand{\NORMRPAR}[1]{\ensuremath{ \left\| #1 \right\|_{\REALS^{\DIMPAR} } } }

\newcommand{\NORMRPARSQ}[1]{\ensuremath{ \left\| #1 \right\|^{2}_{\REALS^{\DIMPAR} } } }

\newcommand{\NORMH}[1]{\ensuremath{ \left\| #1 \right\|_{\HILBERT} } }

\newcommand{\PRODH}[1]{\ensuremath{ \left< #1 \right>_{\HILBERT} } }

\def\VECA{ \mathbf{a} }
\def\EXP{\mathbb{E}} 
\def\TRACE{  { \tt Tr } }  

\def\TITAVEC{  \pmb{ \mathbf{\nu}} }
\def\DIMM{d_{M}}   
\def\DIMMOP{d_{M}}  
\def\SOM{\Omega}    

\def\SAF{\mathscr{F}}   

\def\SINDX{\Theta}

\def\COLP{\mathfrak{B}}
\def\INDX{\theta}
\def\INDXTRUE{\INDX_{T}}
\def\SOMP{\mathscr{P}}
\def\SOMPT{\SOMP_{\INDX}}
\def\SOMPTRUE{\SOMP_{\INDXTRUE}}
\def\SOMPTONE{\SOMP_{\INDX_{1}}}
\def\SOMPTWO{\SOMP_{\INDX_{2}}}
\def\SAMP{\pmb{\mathbf{\mathscr{X}}}}
\def\DIMSAMP{d_{S}}
\def\REALS{\mathbb{R}}
\def\NAT{ \mathbb{N} }
\def\RN{\REALS^{\DIMSAMP}}
\def\BN{\mathscr{B}_{\DIMSAMP}}
\def\RPBB{\mathbb{P}}
\def\RPT{\RPBB_{\INDX}}
\def\RPTRUE{\RPBB_{\INDXTRUE}}

\def\PSPACET{ (\RN, \BN, \RPT) }

\def\PSPACETRUE{ (\RN, \BN, \RPTRUE) }
\def\SPACELONE{\LEB_{1} \PSPACET}

\def\SPACELTWOTRUE{\LEB_{2} \PSPACETRUE}
\def\SPACELTWOT{\LEB_{2} \PSPACET}
\def\DIFF{ {\rm d} }
\def\LV{\mathscr{L}}
\def\LVONE{{\LV}_{1}}
\def\LVTWO{{\LV}_{2}}
\def\VSPACEONE{ \LVONE \PSPACET}

\def\VSPACETWOTRUE{ \LVTWO \PSPACETRUE}

\def\PAR{P}
\def\DIMPAR{d_{\PAR}}
\def\RPAR{\REALS^{\DIMPAR}}
\def\BORELS{{\mathscr{B}}}
\def\BPAR{\BORELS_{\DIMPAR}}
\def\FI{\pmb{\mathbf{\varphi}}}

\def\INDXONE{\INDX_{1}}
\def\INDXTWO{\INDX_{2}}
\def\RPTONE{\RPBB_{\INDX_{1}}}
\def\RPTTWO{\RPBB_{\INDX_{2}}}
\def\SOMPTONE{\mathscr{P}_{\INDX_{1}}}
\def\SOMPTTWO{\mathscr{P}_{\INDX_{2}}}
\def\PI{\pi}
\def\PIT{\PI ({\INDX}) }
\def\PITI{\PI ({\INDX_{i}})}
\def\PISUBI{  \PI_{i} }
\def\LPISUBI{ \OPZ ( \PISUBI ) }

\def\PIPJ{\PI^{\prime}_{j}}
\def\LPIPJ{ \OPZ ( \PIPJ ) }
\def\XVAR{{\mathbf{x}}}
\def\FULLPIT{\PI_{\INDX}(\XVAR,\INDXTRUE)}
\def\FULLPITSAMP{\PI_{\INDX}(\SAMP,\INDXTRUE)}
\def\FUNCG{{\bf g}}
\def\FUNCH{{\bf h}}
\def\GT{{\FUNCG}(\INDX)}
\def\GTONE{{\FUNCG}(\INDX_{1})}
\def\GTTWO{{\FUNCG}(\INDX_{2})}
\def\HT{{\FUNCH}(\INDX)}
\def\GTZERO{{\FUNCG}(\INDX_{0})}
\def\HTZERO{{\FUNCH}(\INDX_{0})}
\def\GTRUE{{\FUNCG}(\INDXTRUE)}
\def\HTRUE{{\FUNCH}(\INDXTRUE)}

\def\MLAMBDA{\lambda}
\def\pT{p_{\INDX}}
\def\pTRUE{p_{\INDXTRUE}}

\def\LEB{L}
\def\MEAS1{\mu}
\def\DIMP{s}
\def\DIMQ{r}

\def\SPACELQ{\LEB_{\DIMQ} \PSPACETRUE}
\def\SPACELP{\LEB_{\DIMP} \PSPACETRUE}
\def\THREEBETA{\mathfrak{B}_{0}}
\def\BETABK{\mathfrak{B}}
\def\HILBERT{\mathscr{H}}
\def\OP{L}
\def\OPZ{\OP_{\THREEBETA}}
\def\OPS{\OP_{S}}
\def\OPC{\OP_{C}}
\def\KH{K_{H}}
\def\RPLUS{{\REALS}^{+}}

\def\HTI{\FUNCH ({\INDX_{i}})}

\def\ASUBI{a_{i}}
\def\ASUBK{a_{k}}

\def\APJ{a^{\prime}_{j}}
\def\TSUBI{\INDX_{i}}

\def\NORMGTI{\NORMRPAR { \sum_{i = 1}^{\DIMM} \ASUBI \; \HTI } }
\def\NORMGTISQ{\NORMRPARSQ { \sum_{i = 1}^{\DIMM} \ASUBI \; \HTI } }

\def\NORMPITI{\NORMLTWO { \sum_{i = 1}^{\DIMM} \ASUBI \; \PITI  } }
\def\NORMPITISQ{\NORMLTWOSQ { \sum_{i = 1}^{\DIMM} \ASUBI \; \PITI  } }

\def\PIHI{\widehat{\PI}_{i}}
\def\NORMPIH{ {\parallel \sum_{i = 1}^{\DIML} \ASUBI \; \PIHI \parallel_{\HILBERT}} } 
\def\SLPIH{ \sum_{i = 1}^{\DIML} \ASUBI \; \OPZ (\PIHI) }
\def\NORMLPIH{ \parallel  \SLPIH \parallel_{\RPAR}}
\def\VHI{ \widehat{v}_{i} }
\def\WHI{ \widehat{w}_{i} }
\def\SLVH{ \sum_{i = 1}^{\DIML} \ASUBI \; \OPC (\VHI) }
\def\NORMLVH{ \parallel  \SLVH \parallel_{\RPAR}}
\def\NORMVH{ {\parallel \sum_{i = 1}^{\DIML} \ASUBI \; \VHI \parallel_{\HILBERT}} } 
\def\UI{ u_{i} }
\def\UJ{ u_{j} }
\def\UH{ \widehat{u} }
\def\UHI{ \widehat{u}_{i} }
\def\UHJ{ \widehat{u}_{j} }
\def\UHK{ \widehat{u}_{k} }
\def\SUH{ \sum_{i = 1}^{\DIML} \ASUBI \; \UHI }
\def\SLUH{ \sum_{i = 1}^{\DIML} \ASUBI \; \OPC (\UHI) }

\def\NAT{\mathbb{N}}
\def\KERC{\mathscr{N}_L} 
\def\KERPER{\KERC^{\perp}}

\def\IMAGE{ \mathbb{I}  }
\def\DIML{ d_{L} }

\def\IL{ I_{\OP} }
\def\PROJ{ {\rm Proj}}
\def\ALFAIU{ \alpha_{i}(u) }
\def\ALFAKU{ \alpha_{k}(u) }
\def\GAMAU{ w(u) } 
\def\ALFAIUN{\alpha_{i}(u_{n}) }
\def\GAMAIUN{ w(u_{n}) }  
\def\BETAZI{ a_{i}}   
\def\FIH{ \pmb{ \widehat{ \mathbf{\varphi }} }}
\def\FIHC{ \FIH_{c} }
\def\SPACELTWO{\LEB_{2}\PSPACETRUE}
\def\MSUP{ {\tt {msup} } }
\def\COVTRUE{{\tt Cov}_{\INDXTRUE} }

\def\DIMA{d_{A}}
\def\DIMT{d_{\tau}}
\def\TAO{ \pmb{ \mathbf{\tau}}}
\def\TAOH{ \widehat{ \pmb{ \mathbf{\tau}} } }
\def\BETABOLD{ \pmb{ \mathbf{\beta}}}

\def\PSI{ \pmb { \mathbf{\psi }} }
\def\PSIH{ \pmb{ \widehat{ \mathbf{\psi }} }}
\def\PSIHC{ \PSIH_{c} }

\def\DET{ {\tt Det} }
\def\BFG{ {\bf g}}
\def\UG{ \mathscr{U}_{\BFG} }
\def\ISN{ \; {\rm i.s.n.}}
\def\WP1{ \; {\tt w.p. } \, 1 }

\def\KRON{ \delta }

\def\VECB{ \BETABOLD }
\def\QUADQ{ {\bf q} }
\def\PAIRD{ {\bf d} }
\def\CONDA{\mathscr{C}_{A}}
\def\CONDB{\mathscr{C}_{B}}
\def\COLLW{\mathscr{W}}
\def\WA{\COLLW_{A}}
\def\WB{\COLLW_{B}}

\def\EXPTRUE{ {\EXP_{ \INDXTRUE } \! \! } }   
\def\EXPT{ {\EXP_{ \INDX } } }   
\def\MATRIXL { \mathbb{L} }
\def\AVEC{ {\bf a} }
\def\DIAG{ {\tt Diag} }
\def\UVECH{ {\widehat{{\bf u}}} }
\def\IDL{ I_{\DIML} }
\def\IDP{ I_{\DIMPAR} }

\def\IDMP{ I_{\DIMM^{\prime}} }
\def\YIM{ {\widehat{s}}_{i}(m) }
\def\SVECH{ {\widehat{{\bf s}}} }
\def\MATRIXS{ \mathbb{S} }
\def\QUADQH{ \widehat{\QUADQ} }
\def\GAMMAVEC{ \pmb{ \mathbf{\gamma}}}
\def\RHOVEC{ \pmb{ \mathbf{\rho}}}

\def\DIMG{ d_{\gamma} }
\def\DIMR{ d_{\rho} }
\def\MATRIXM { \mathbb{M}  }
\def\MATRIXLAM{ \Lambda }
\def\MU{\mu}

\def\UVEC{ \mathbf { u } }

\def\XVEC{ \mathbf { x } }

\def\FLOP{\alpha}
\def\PLOP{\gamma }
\def\ALFAVEC{  \pmb{ \mathbf{\alpha}} }

\def\MATPSI{ \Psi }
\def\LMATX{  X }
\def\SMATF{  F }
\def\LMATY{  Y }
\def\LMATH{  H }
\def\LMATG{  G }
\def\LMATA{  A }
\def\LMATB{  B }
\def\LMATT{  T }
\def\LMATR{  R }
\def\LMATD{  D }
\def\LMATL{  \Lambda }
\def\DIMLN{ N }
\def\DIMLM{ M }
\def\DIMLP{ P }
\def\DIMMAT { M}
\def\MATBOUND{ B_{\COLLW} }

\begin{document}

\title{Barankin Vector Locally Best Unbiased Estimates\thanks{This work was partially supported by the Universidad de Buenos Aires, UBA, grant UBACYT No. 20020130100357BA, and the Consejo Nacional de Investigaciones Cient{\'\i}ficas y T\'ecnicas, CONICET, Argentina.}}   
\author{Bruno~Cer\hspace{1pt}nuschi-Fr{\'\i}as\thanks{B. Cer\hspace{1pt}nuschi-Fr{\'\i}as is with the Universidad de Buenos Aires, Facultad de Ingenier{\'\i}a,   
and the Instituto Argentino de Matem\'aticas, IAM, CONICET, Casilla 8, Sucursal 12(B), 1412 Buenos Aires, Argentina (e-mail: bcf@acm.org).}}
\maketitle         
\markboth{B. Cernuschi-Fr{\'\i}as} {B. Cernuschi-Fr{\'\i}as}

\begin{abstract}
The Barankin bound is generalized to the vector case in the mean square error sense. Necessary and sufficient conditions are obtained to achieve the lower bound. To obtain the result, a simple finite dimensional real vector valued generalization of the Riesz representation theorem
for Hilbert spaces is given. The bound has the form of a linear matrix inequality where the covariances of any unbiased estimator, if these exist, are lower bounded by matrices depending only on the parametrized probability distributions.
\end{abstract}


\section{Introduction} \label{sec:INTRO}
The problem considered, following Barankin, \cite{BARANKIN}, and results in Banach, \cite{BANACH}, is the optimal unbiased estimation of a deterministic vector of parameters $ \TITAVEC $ of a family of probability measures $ \SOMP_{\TITAVEC} $,
or more generally a known real vector function of these parameters $ \FUNCG ( \TITAVEC ) $,
using a realization of a vector random variable $ \SAMP $ drawn from $ \SOMP_{\TITAVEC_{T}} $.
The first issue is to find a function $ \PSI $ such that $ \int \PSI ( \SAMP ) \; \DIFF \SOMP_{\TITAVEC} = \FUNCG ( \TITAVEC ) $, for all
$ \TITAVEC $ in some admissible set.
This problem is a vector integral equation and may or may not have a solution, 
\cite{BITSADZE, SAITOH-2016}.
Furthermore, even if it has solution, it may not have a solution with finite covariance matrix for $ \TITAVEC_{T} $.
Barankin, under very simple hypothesis, \cite{BARANKIN}, gives an if and only if condition for the existence of a 
minimal $ \DIMP $-th variance unbiased estimator for the scalar case,
which is tighter than the classical Cramer-Rao or Bhattacharyya bounds if they exist.
In recent years the Barankin bound has attracted attention, since there are problems for which the Cramer-Rao or Bhattacharyya bounds
give no satisfactory solution, see e.g. 
\cite{VANTREES}, 
and references there.
Following \cite{BARANKIN}, the problem studied here is under what conditions there exists a finite covariance vector unbiased estimator of the true vector parameter $ \TITAVEC_{T} $, and in that case
if a minimal covariance vector unbiased estimator exists.

In Section \ref{sec:FORM} an overview is presented of the relevant results of measure theory and the Lebesgue integral related to the Barankin formulation. In Section \ref{sec:MB} the vector Barankin bound generalization is presented as a linear matrix inequality (LMI). In Section \ref{sec:FAS} the Barankin functional analysis formalization is generalized to handle the vector case.
In Section \ref{sec:GRRT} a finite dimensional real vector valued generalization of the Riesz representation theorem for Hilbert spaces is presented. In Section \ref{sec:OPTEST} necessary and sufficient conditions are given for the existence of an optimal vector estimator attaining the bound given by the LMI obtained in Section \ref{sec:MB}. In Section \ref{sec:OLMI} other alternative LMI formulations for the existence of an optimal vector estimator are given.
\section{ Formalization of the vector estimation problem}    \label{sec:FORM}
\subsection{Measure theoretic setup} \label{subsec:MTS}
Let $ ( \SOM, \SAF ) $ be a measurable space, where $ \SOM $ is a well defined abstract set, and $ \SAF $ is a sigma-algebra of subsets of $ \SOM $, \cite{HALMOS}. 
Let $ \SINDX $ be an abstract arbitrary set of sub-indexes with no conditions on its structure as in \cite{BARANKIN}, p. 477.
Let $ \COLP $ be a collection of probability measures $ \SOMPT $ for the measurable space $ ( \SOM, \SAF ) $, 
indexed by the sub-indexes $ \INDX \in \SINDX $, i.e.  $ \COLP = \bigl\{ \SOMPT: \INDX \in \SINDX \bigr\} $, as in \cite{BARANKIN} p. 477.
Hence for each $ \INDX \in \SINDX $, the triple $( \SOM, \SAF, \SOMPT ) $ is a probability space. 
Let $ \SAMP $ be a vector random variable, i.e. a measurable function from the measurable space $ ( \SOM, \SAF ) $ 
to the measurable space $ ( \RN$, $\BN ) $, where $ \RN $ is the vector $ \DIMSAMP $-dimensional real space, and $ \BN $ is the Borel sigma-algebra for $ \RN $, that is the minimal sigma-algebra generated, e. g., by the open sets of $\RN$. 
Then $ \SAMP $ is a real vector random variable iff $ \forall B \in \BN $ we have  $ \SAMP^{-1}(B) \in \SAF $, if and only if each component 
of the vector is a real random variable, \cite{LAHA} p. 19.
Define for each $ \INDX\ \in \SINDX $ the measure $ \RPT $, for the measurable space $ ( \RN$, $\BN ) $, induced by the random variable $ \SAMP $, \cite{CHUNG} p. 34, i.e. for each $ B \in \BN $ define $ \RPT (B) = \SOMPT (\SAMP^{-1} (B)) $. 
Hence for each $ \INDX \in \SINDX $ the random variable $ \SAMP $ induces the probability
space $ ( \RN$, $\BN$, $\RPT ) $. 

Let $ \PSI $ be a real measurable vector function from $ ( \RN$, $\BN ) $ to $ ( \RPAR$, $\BPAR ) $, that is $ \PSI: \RN \to \RPAR $, 
and for each $ B \in \BPAR $ we have $ \PSI^{-1}(B) \in \BN $. For the measurable vector function  $ \PSI $ from $ ( \RN, \BN  ) $ to $ ( \RPAR, \BPAR ) $, define the i-th component of the vector $ \PSI $
as $ \left[ \PSI \right]_{i} $, which is a measurable function from $ ( \RN, \BN  ) $ to $ ( \REALS, \BORELS ) $ for each $ 1 \le i \le \DIMPAR $
iff $ \PSI $ is measurable.
Define $ \VSPACEONE $ as the collection of all the measurable vector functions $ \PSI $ from $ ( \RN, \BN  ) $ to $ ( \RPAR, \BPAR ) $, such that
$ \int \bigl| \left[ \PSI \right]_{i} \bigr| \DIFF \RPT < + \infty $, for $ 1 \le i \le \DIMPAR$, equivalently, 
$ \left[ \PSI \right]_{i} \in \SPACELONE $, for all $ 1 \le i \le \DIMPAR$,
so that $ \VSPACEONE = \left( \SPACELONE \right)^{\DIMPAR} $.

Hence $ \PSI(\SAMP) $ is a random variable from $ ( \SOM, \SAF ) $ to $ ( \RPAR , \BPAR ) $, since for $ B \in \BPAR $ we have
$ \bigl[ \PSI ( \SAMP ) \bigr]^{-1} ( B ) = \SAMP^{-1} \bigl( \PSI^{-1} ( B ) \bigr) $, but $ B \in \BPAR $ so that $ \PSI^{-1} ( B ) \in \BN $ and then $ \SAMP^{-1} \bigl( \PSI^{-1} ( B ) \bigr)  \in \SAF $.

Define the integral of a vector of functions
as a vector  whose elements are the integrals of each function. Then, \cite{CHUNG} p. 45:
\begin{displaymath}
\int \PSI(\SAMP) \; d\SOMPT = \int \PSI \; d\RPT
\end{displaymath}
Note that the integral on the left is with respect to the probability space ($\SOM$,$\SAF$,$\SOMPT$), while the integral on the right is with respect to the probability space ($\RN$, $\BN$, $\RPT$). 
We will refer indistinctly to $ \PSI ( \SAMP ) $ and $ \PSI $ as an {\em estimator}, with the understanding that they refer to different probability spaces linked by the previous equality of integrals.

For $ \PSI \in \VSPACEONE $ define the expectation of $ \PSI $ as
$ \EXPT \left[ \PSI \right] =$\break  $\int \PSI(\SAMP) \; d\SOMPT = \int \PSI \; d\RPT $. 

We assume that the random variable $\SAMP$ is drawn from some specific probability measure (p. m.) $\SOMPTRUE$,
with $ \INDXTRUE \in \SINDX $, i. e. we will use the realization of this random variable to obtain the estimator for $ \GTRUE $.
The random vector $ \PSI (\SAMP) $ is an unbiased estimator for $\GT$, $ \forall \INDX \in \SINDX $ , $ \FUNCG: \SINDX \to \RPAR$, if  the integral
$\int \PSI (\SAMP) \; \DIFF  \SOMPT $ is well defined, $ \forall \INDX \in \SINDX $,  and we have
$
\int \PSI (\SAMP) \; \DIFF  \SOMPT = \int \PSI \; \DIFF \RPT =  \GT 
$,
$ \forall \INDX \in \SINDX $. 

Then the first issue posed in the introduction may be formally stated as:

\begin{PROBLEM}[Basic Problem] \label{prob:INTEQ}
Given a function $ \FUNCG: \INDX \to \RPAR$, defined for each $ \INDX \in \SINDX $,
and a family of p.m.'s indexed by $ \INDX \in \SINDX $, 
find an unbiased estimator, i.e. find a function $\PSI \in \VSPACEONE$, $ \forall \INDX \in \SINDX $, 
such that $ \int \PSI ( \SAMP )\; \DIFF \SOMPT = \int \PSI \; \DIFF \RPT = \GT $, for all $\INDX \in \SINDX$.
\end{PROBLEM}
Define the integral of a matrix $ \MATPSI $ of dimensions $ N \times M $, $ N, M \in \NAT $, 
whose elements belong to $ \SPACELONE $, 
as a matrix whose elements are the integrals of the elements of $ \MATPSI $,
so that
$ \EXPT [ \MATPSI ( \SAMP ) ] = \int \MATPSI ( \SAMP ) \; \DIFF \SOMPT = \int \MATPSI \; \DIFF \RPT $.
For a measurable square integrable function $ f: ( \RN, \BN, \RPTRUE  ) \to( \REALS, \BORELS) $, i.e. $  f \in \SPACELTWO $,
\ifundefined{DRAFT}
\else
\\
\noindent
\fi
Define $ \VSPACETWOTRUE $ as the collection of all the measurable functions $ \UVEC $ from $ ( \RN, \BN  ) $ to $ ( \RPAR, \BPAR ) $, such that
$ \int \left[ \UVEC \right]_{i}^{2} \DIFF \RPTRUE < + \infty $, for $ 1 \le i \le \DIMPAR$, equivalently, $ \left[ \UVEC \right]_{i} \in \SPACELTWOTRUE  $, for all $ 1 \le i \le \DIMPAR$. 
If the non-centered second order moments of the components of the estimator $ \PSI $  exist for 
$ \INDXTRUE$, i.e. $ \left[ \PSI \right]_{i} \in \SPACELTWOTRUE  $, for $ 1 \le i \le \DIMPAR$,
so that $ \PSI \in \VSPACETWOTRUE \equiv \left( \SPACELTWOTRUE  \right)^{\DIMPAR} $,
then, the first order moments of the components of the estimator exist for $ \INDXTRUE$. Also, the  correlations 
$ \int \left[ \PSI \right]_{i}  \left[ \PSI \right]_{j} \DIFF \RPTRUE $,
are well defined and are finite for all $ i \neq j$, $ 1 \le i,j \le \DIMPAR $, and using the Cauchy-Schwarz inequality, we obtain 
$ \Big| \int \left[ \PSI \right]_{i}  \left[ \PSI \right]_{j} \DIFF \RPTRUE \Bigr|
\le
\NORMLTWO{ \left[ \PSI \right]_{i} } \, \NORMLTWO{ \left[ \PSI \right]_{j} } 
$. 
Additionally assume  $ \PSI (\SAMP) $ is unbiased $ \forall \INDX \in \SINDX $,
then the covariance matrix of $ \PSI (\SAMP) $ exists for $ \INDXTRUE$, and we have
$ 
\COVTRUE(\PSI) = \int \left[ \left( \PSI (\SAMP) - \GTRUE \right) \; \left( \PSI (\SAMP) - \GTRUE \right)^{T} \right] \DIFF \SOMPTRUE
=
$\break
$
\EXPTRUE \left[ \left( \PSI - \GTRUE \right) \; \left( \PSI - \GTRUE \right)^{T} \right] 
=
\EXPTRUE \left[ \PSI \; \PSI^{T} \right] - \GTRUE \; \GTRUE^{T}
$.

In the same direction of \cite{BARANKIN}, with $ \DIMP = \DIMQ = 2 $, instead of the general Problem \ref{prob:INTEQ},
we pose the problem in terms of estimators with finite covariance matrix at $ \INDXTRUE $:
\begin{PROBLEM}[Finite Covariance Problem]     \label{prob:LTWO}
Given a function $ \FUNCG: \INDX \to \RPAR$, defined for each $ \INDX \in \SINDX $,
and a family of p.m.'s indexed by $ \INDX \in \SINDX $,
find a function $\PSI \in \VSPACETWOTRUE $, with $\PSI \in \VSPACEONE $, $ \forall \INDX \in \SINDX $, such that $ \int \PSI \; \DIFF \RPT = \GT $, for all $\INDX \in \SINDX$.
If there are several solutions find, if possible, a solution with minimal covariance matrix at $ \INDXTRUE $.
\end{PROBLEM}
\subsection{Centered definitions}
Define $ \FI = \PSI - \GTRUE $, and $ \HT = \GT - \GTRUE $ so that $ \HTRUE = 0 $. If $ \PSI $ is unbiased, then, 
since
$
\int \PSI \; \DIFF \RPT= \GT
$
and
$
\int \GTRUE \;  d \RPT = \GTRUE
$, for all $ \INDX \in \SINDX $, 
then,
$
\int \left ( \PSI  - \GTRUE \right) \; d \RPT = \GT - \GTRUE
$, for all $ \INDX \in \SINDX $, so that
$
\int \FI \; \DIFF \RPT = \HT \quad \forall \INDX \in \SINDX
$, 
and
$ \COVTRUE(\PSI) = \EXPTRUE \left[ \FI \: \FI^{T} \right] $.
Also, if $ \PSI $ is unbiased, then, since $ \HTRUE = 0, $ then $ \EXPTRUE \left[ \: \FI \right] = \int \FI \; \DIFF \RPTRUE = 0 $.
\subsection{Barankin formulation: basic hypothesis}
Following Barankin we will introduce some simple
additional hypothesis resumed in Barankin's Postulate in \cite{BARANKIN} p. 481.
\begin{HYPOTHESIS}    \label{hypo:ONE}
The set $ \SINDX $ is an arbitrary index set with no conditions on its structure, \cite{BARANKIN} p. 477, and 
$ \COLP $ is a collection of probability measures $ \SOMPT $ for the measurable space $ ( \SOM, \SAF ) $,
i.e.  $ \COLP = \bigl\{ \SOMPT: \INDX \in \SINDX \bigr\} $ as in  \cite{BARANKIN}, p. 477.
The random variable $\SAMP: ( \SOM, \SAF ) \to ( \RN, \BN ) $ is drawn from the probability measure (p. m.) $\SOMPTRUE$,
with $ \INDXTRUE \in \SINDX $.
Assume that for each $\INDX \in \SINDX$ the p.m. $\SOMPT$ is absolutely continuous with respect to $\SOMPTRUE$, i.e $\SOMPT << \SOMPTRUE$,
with $ \INDXTRUE \in \SINDX $. 
\end{HYPOTHESIS}
\begin{LEMMA} \label{lemma:abscont}
If Hypothesis \ref{hypo:ONE} is true then for each $\INDX \in \SINDX$ the p.m. $\RPT$ is absolutely continuous with respect to $\RPTRUE$, i.e. $\RPT << \RPTRUE$.
\end{LEMMA}
\begin{proof}
Assume $B\in \BPAR$ is such that $\RPTRUE (B) = 0$, then since $\RPTRUE (B) = \SOMPTRUE(\SAMP^{-1}(B))$, we obtain $\SOMPTRUE(\SAMP^{-1}(B)) =0$. But $ \SOMPT << \SOMPTRUE $, hence
$\SOMPT(\SAMP^{-1}(B)) = 0$. Since $\RPT (B) = \SOMPT(\SAMP^{-1}(B))$, then $\RPT (B) = 0$.
\end{proof}
\begin{OBSERVATION} \label{obs:abscont}
In the case in which every index $\INDX \in \SINDX$ is a possible candidate for $\INDXTRUE$, then, Hypothesis \ref{hypo:ONE} should require that for each $\INDXONE\in \SINDX$ the p.m. $\SOMPTONE$ should be absolutely continuous with respect to each other p.m. $\SOMPTTWO$
with $\INDXTWO \in \SINDX$, and then $\RPTONE << \RPTTWO$ for all $\INDXONE, \INDXTWO \in \SINDX$.
\end{OBSERVATION}
As a consequence of the previous hypothesis and lemma, the Radon-Nykodim derivatives $d \SOMPT / d \SOMPTRUE$ and $d \RPT / d \RPTRUE$ exist for all $\INDX \in \SINDX$, \cite{HEWITT} p. 315.
\begin{DEFINITION} \label{def:PI}
Define $\PIT = d \RPT / d \RPTRUE$, with $\PIT \equiv \FULLPIT$, $\XVAR \in \RN$, so
that 
$
d \SOMPT / d \SOMPTRUE \; = \; \FULLPITSAMP
$.  
Define $ \THREEBETA = \left\{ \PIT:  \INDX \in \SINDX \right\}$, see \cite{BARANKIN}, p. 481.
\end{DEFINITION}
We have $ \PI ( \INDX ) \ge 0 $ w.p. 1, for all $ \INDX \in \SINDX $, \cite{HEWITT} p. 315,  $ \PI ( \INDXTRUE ) = 1 $ w.p. 1, and  
$ \int  \FULLPITSAMP \DIFF \SOMPTRUE = \int \bigl( \DIFF \SOMPT / \DIFF \SOMPTRUE   \bigr) \DIFF \SOMPTRUE = 
\int \DIFF \SOMPT = 1 $, for all $ \INDX \in \SINDX $. 
\begin{HYPOTHESIS}      \label{hypo:BASIC}
\begin{enumerate}
\item \label{it:IDENT}
Assume that for each $\INDX$ there is one and only one $\PIT \in \THREEBETA $, 
i.e. the correspondence $ \PI: \SINDX \to  \THREEBETA $ is one-to-one.
\item \label{it:NONC}
There are at least two  values $ \INDX_{1},  \INDX_{2} \in \SINDX $, such that
$ \GTONE \neq \GTTWO $.
\end{enumerate}
\end{HYPOTHESIS}
\begin{OBSERVATION}       \label{obs:BASIC}
Item \ref{it:IDENT} avoids the identifiability problem, \cite{DUDA}, pp. 58 and 191.
Item \ref{it:NONC} implies that we do not consider estimators which are constant with probability 1: if it was
$ \PSI = \ALFAVEC_{0} $ w.p. 1 for some $ \ALFAVEC_{0} \in \REALS^{\DIMPAR} $, then  
$ \int \PSI (\SAMP) \; \DIFF  \SOMPTONE =  \int  \ALFAVEC_{0} \; \DIFF  \SOMPTONE = \ALFAVEC_{0} $, similarly
$ \int \PSI (\SAMP) \; \DIFF  \SOMPTWO = \ALFAVEC_{0} $, but since we assume that $ \PSI $ is unbiased, it should be $ \int \PSI (\SAMP) \; \DIFF  \SOMPTONE = \GTONE $ and $ \int \PSI (\SAMP) \; \DIFF  \SOMPTTWO = \GTTWO $, and then it should be $ \GTONE = \ALFAVEC_{0} = \GTTWO $, which is a contradiction. Additionally, Hypothesis \ref{hypo:BASIC} implies that there exists at least a $ \INDX_{0} \in \SINDX $ such that 
$ \GTZERO \neq 0 $. Nonetheless, see e.g. \cite{BARANKIN} p. 482 and \cite{DMC-1} p. 2440,  for some comments regarding constant estimators.
\end{OBSERVATION}
\subsection{Barankin postulate} 
The following hypothesis is Barankin's Postulate in \cite{BARANKIN}, p. 481, for $ \DIMP = \DIMQ = 2 $.
\begin{HYPOTHESIS}[Barankin, \cite{BARANKIN}, Postulate p. 481]  \label{hypo:TWO}
Assume that 
\begin{equation} 
\nonumber
 {\PIT \in \SPACELTWOTRUE } \qquad \forall \INDX \in \SINDX 
\end{equation}
equivalently $ \THREEBETA \subseteq \SPACELTWOTRUE $.
\end{HYPOTHESIS}
\begin{OBSERVATION}    \label{obs:THREEBETA}
Since $ \int \PIT \; \DIFF \RPTRUE = \int \bigl( \DIFF \RPT / \DIFF \RPTRUE \bigr) \; \DIFF  \RPTRUE = 
\int \DIFF \RPT = 1 $, for all $ \INDX \in \SINDX $, then $ \NORMLTWO { \PIT } \neq 0  $, for all $ \INDX \in \SINDX $,
equivalently $ \NORMLTWO { u } \neq 0  $, for all $ u \in \THREEBETA $. If not,
$ \NORMLTWO { \PIT } = 0  $ implies $ \PIT = 0 $ w.p. 1, and then $ \int \PIT \; \DIFF \; \RPTRUE = 0 $,  contradiction.
Additionally note, taking in account Hypothesis \ref{hypo:BASIC}, that $ \THREEBETA $ has at least two elements.
\end{OBSERVATION}
Suppose $ \PSI   \in  \VSPACETWOTRUE $, since $ \PIT \in \SPACELTWOTRUE $ for all $ \INDX \in \SINDX $,
then the integrals $ \int \PSI \; \PIT \; d \RPTRUE $ are well defined 
for all $ \INDX \in \SINDX $, and we have all the equivalent forms:
\begin{align}
\nonumber
& \int \PSI \; \PIT \; \DIFF  \RPTRUE = \int \PSI \; \frac{ \DIFF \RPT}{\DIFF \RPTRUE} \; \DIFF \RPTRUE = \int \PSI \; \DIFF \RPT  
\\
\nonumber
&= \int \PSI(\SAMP) \; \DIFF \SOMPT = \int \PSI(\SAMP) \frac{ \DIFF \SOMPT}{ \DIFF \SOMPTRUE} \; \DIFF \SOMPTRUE
=
\EXPTRUE \left[ \: \PSI \; \PIT \right]
\end{align}
If $ \PSI \in \VSPACETWOTRUE $ is unbiased, then for $ \FI = \PSI - \GTRUE $ we have
\begin{equation}     \label{eq:UNBIASED}
\EXPTRUE \left[ \, \FI  \; \PIT \right] = \int \FI \; \PIT \; \DIFF  \RPTRUE = \HT \qquad \forall \; \INDX \in \SINDX 
\end{equation}
The introduction of the functions $ \PI $ reduces the consideration of the 
multiple probability spaces
$ \SPACELTWOT$, $ \forall \INDX \in \SINDX $, to a single probability space $ \SPACELTWOTRUE $.
\subsection{Probability density function form}
Call $ \MLAMBDA $ the Lebesgue measure for the measurable space ($\RN$, $\BN$), i.e. the measure that assigns to parallelepipeds in $\RN$ the value given by the product of the lengths of the edges of the parallelepiped in each direction. 
Alternatively call $ \DIFF \MLAMBDA \; = \; \DIFF \XVAR$, with $ \XVAR \in \RN$. 
If in turn
we have $ \RPTRUE << \MLAMBDA $, i.e. the p.m. $ \RPTRUE $ is absolutely continuous with respect to the Lebesgue measure, then, $ \RPT << \RPTRUE << \MLAMBDA $, 
so that $ \RPT << \MLAMBDA $, and  then the  Radon-Nykodim derivatives $d \RPT / d \MLAMBDA$ exist, for all $ \INDX \in \SINDX $.
These derivatives are the probability density functions (pdf) $\pT \equiv \pT(\XVAR) \equiv d \RPT / d \MLAMBDA$ with $\XVAR \in \RN$.
Since, \cite{HEWITT} p. 328,
\begin{displaymath}
 \pT = \frac{ d \RPT }{ d \MLAMBDA} = \frac {d \RPT }{ d \RPTRUE} \; \frac {d \RPTRUE }{ d \MLAMBDA} =  \PIT \; \pTRUE  \qquad \MLAMBDA{\rm-ae}
\end{displaymath}
then, if $ \PSI $ is unbiased
\begin{displaymath}
\GT = \int \PSI \; d \RPT = \int \PSI \; \pT \; d \MLAMBDA = \int \PSI \; \PIT \; \pTRUE \; d \MLAMBDA
\end{displaymath}
\subsection{The Main Problem}
With all the previous considerations we may formalize the generalization to the vector case of the Barankin formulation as:

\begin{PROBLEM}[Main Problem]     \label{prob:MAIN}
Given a function $ \FUNCG: \INDX \to \RPAR$, defined for each $ \INDX \in \SINDX $,
and a family of p.m.'s indexed by $ \INDX \in \SINDX $, that satisfy
the Hypothesis \ref{hypo:ONE}, \ref{hypo:BASIC}, and \ref{hypo:TWO},
find a function $ \PSI \in \VSPACETWOTRUE $, such that $ \int \PSI \; \PIT \; \DIFF \RPTRUE = \GT $, $ \forall \INDX \in \SINDX $.
If there are several solutions find, if possible, a solution with minimal covariance matrix at $ \INDXTRUE $.
\end{PROBLEM}
The solution to this problem is given below in Theorem \ref{theo:MAIN}.
\section {Matrix bound} \label{sec:MB}
For a vector $ \VECA $ in a finite vector space denote $ \left[ \VECA \right]_{i} $ the i-th component of the vector.
For a matrix $A$, define $ \left[ A \right]_{i} $ as the i-th column of the matrix and $ \left[ A \right]_{i,j} $ as i-th, j-th element of
the matrix. We have $ \left[ A \right]_{i,j} = \left[  \left[ A \right]_{j} \right]_{i} $. Denote $ A^{T} $ the transpose of
the matrix $A$, $ \DET  (A) $ the determinant of $ A $, and $ \TRACE \left[ A \right] $
the trace of $A$.
A square symmetric real matrix $ A \in \REALS^{ N \times N } $ is a symmetric non-negative definite (s.n.n.d.) matrix iff, $ \XVEC^{T} A \XVEC \ge 0 $, for all $ \XVEC \in \REALS^{N} $.
A real s.n.n.d. matrix $ A \in \REALS^{ N \times N } $ is a symmetric positive definite (s.p.d.) matrix if $ \DET ( A) \neq 0 $, iff $ \XVEC^{T} A \XVEC > 0 $, for all $ \XVEC \ne 0 $. 
Two s.n.n.d. matrices $ A \in \REALS^{ N \times N } $ and $ B \in \REALS^{ N \times N } $  are comparable in the L\"owner partial order,
\cite{ZHANG99} p. 166, if either $ A - B $ is s.n.n.d. and then $ A \ge B $, or if $ B - A $ is s.n.n.d. and then $ B \ge A $, else, they are not comparable. 
For $ A $, $ B $ and $ C $ s.n.n.d of dimensions $ N \times N $, then if $ A \ge B $ and $ B \ge C $ then $ A \ge C $,
and if $ A \ge B $ and $ B \ge A $, then $ A = B $, see e.g. \cite{DMC-1} Lemma 3  p. 2444.
If $ A \in \REALS^{ N \times N } $ is s.p.d., and $ S \in \REALS^{ N \times M } $ is arbitrary, such that 
$ S^{T} A S = 0 $, then $ S = 0 $, see e.g. \cite{DMC-1} Lemma 2 p. 2444.
For $ A \in \REALS^{ N \times N } $ denote the Frobenius norm as $ \NORMFROB { A } = \bigl( \TRACE ( A A^{T} ) \bigr)^{1/2} $.
 
The following lemma is a variant of the information inequality \cite{ZHANG-2005} p. 172, \cite{GORMAN} Lemma 1 p. 1288, \cite{RAO} pp. 326--328.
\begin{LEMMA}  \label{lemma:SCHUR}
Let $ ( X, {\bf X}, \MU ) $ be an arbitrary measure space. Let $ L_{2}( X, {\bf X}, \MU ) $ be the collection of
all the measurable square integrable real valued 
functions from $X$ to $\REALS$.
Let $ \DIMG, \DIMR, \DIMA \in \NAT $, $ \GAMMAVEC \in \left({\LTWO( X, {\bf X}, \MU )}\right)^{\DIMG}$, 
$ \RHOVEC \in \left({\LTWO( X, {\bf X}, \MU )}\right)^{\DIMR}$,
and $ A \in \REALS^{ \DIMA \times \DIMR } $.
Call $ \SMATF = \int \GAMMAVEC  \; \RHOVEC^{T} \; \DIFF \MU $, $ \SMATF \in \REALS^{ \DIMG \times \DIMR } $,
and $ B = \int \RHOVEC  \; \RHOVEC^{T} \; \DIFF \MU $, $ B \in \REALS^{ \DIMR \times \DIMR } $.
If $ \DET \left( A \; B \; A^{T} \right) \neq 0 $, then
$
\int \GAMMAVEC \; \GAMMAVEC^{T} \; \DIFF \MU
\ge 
 \SMATF \; A^{T} \; \left( A \; B \; A^{T} \right)^{-1}  \; A \;  \SMATF^{T}
$, with equality if and only if
there exists a matrix $ \MATRIXLAM_{0} \in \REALS^{ \DIMG \times \DIMA } $ such that
$ \GAMMAVEC =  \MATRIXLAM_{0} \; A \; \RHOVEC $ $\MU$-almost-everywhere ($\MU$-ae),
and in that case, it is $  \MATRIXLAM_{0} = \SMATF \; A^{T} \left( \; A \; B \; A^{T} \right)^{-1}$.
\end{LEMMA}
\begin{proof}
For each $ \MATRIXLAM \in \REALS^{ \DIMG \times \DIMA }$, let $ \MATRIXM ( \MATRIXLAM ) = 
\int \left( \GAMMAVEC -  \MATRIXLAM \, A  \, \RHOVEC \right) \left( \GAMMAVEC -  \MATRIXLAM \, A \, \RHOVEC \right)^{T} \, \DIFF \MU $, 
$ \MATRIXM ( \MATRIXLAM ) \in \REALS^{ \DIMG \times \DIMG }$.
Then $ \MATRIXM ( \MATRIXLAM ) $ is s.n.n.d. for all $ \MATRIXLAM \in \REALS^{ \DIMG \times \DIMA }$.
We have $  \MATRIXM ( \MATRIXLAM ) =
\int \GAMMAVEC \; \GAMMAVEC^{T} \; \DIFF \MU
-
\SMATF  \; A^{T} \; \MATRIXLAM^{T}
-
\MATRIXLAM \; A \;  \SMATF^{T} \; +
\MATRIXLAM \; A \; B \; A^{T} \; \MATRIXLAM^{T}
$. 
By assumption $ \DET \left( A \; B \; A^{T} \right) \neq 0 $, so that the matrix $ A \; B \; A^{T} $  is invertible.
Define $ \MATRIXLAM_{0} = \SMATF \; A^{T} \; \left( A \; B \; A^{T} \right)^{-1} $,  
then, 
$
\MATRIXM ( \MATRIXLAM_{0} ) = 
\int \GAMMAVEC \; \GAMMAVEC^{T} \; \DIFF \MU
-
 \SMATF  \; A^{T} \; \left( A \; B \; A^{T} \right)^{-1}  \; A \;  \SMATF^{T}
\ge
0
$
so that
$
\int \GAMMAVEC \; \GAMMAVEC^{T} \; \DIFF \MU
\ge 
 \SMATF  \; A^{T} \; \left( A \; B \; A^{T} \right)^{-1}  \; A \;  \SMATF^{T}
$,
and there is equality iff $ \MATRIXM ( \MATRIXLAM_{0} )  = 0 $. 
From the definition of
$ \MATRIXM ( \MATRIXLAM )$, if there exists $ \MATRIXLAM^{\star} \in \REALS^{ \DIMG \times \DIMA }$ such that
$ \GAMMAVEC =  \MATRIXLAM^{\star} \; A \; \RHOVEC $ $\MU$-ae, then
$ \MATRIXM ( \MATRIXLAM^{\star} ) = 0 $.
In that case it will be
$ \int \GAMMAVEC \; \RHOVEC^{T} \DIFF \MU =  \MATRIXLAM^{\star} \; A \; \int \RHOVEC \; \RHOVEC^{T} \DIFF \MU
$, so that $ \SMATF = \MATRIXLAM^{\star} \; A \; B$, and then $ \SMATF \; A^{T} = \MATRIXLAM^{\star} \; A \; B \; A^{T}$.
Since by hypothesis  $ A \; B \; A^{T} $ is invertible, then $  \MATRIXLAM^{\star} = \SMATF \; A^{T} \left( \; A \; B \; A^{T} \right)^{-1} = \MATRIXLAM_{0} $
so that $ M ( \MATRIXLAM_{0} ) = M ( \MATRIXLAM^{\star} ) = 0 $, and then we obtain the equality.
Conversely, if 
$
\int \GAMMAVEC \; \GAMMAVEC^{T} \; \DIFF \MU
= 
\SMATF \; A^{T} \; \left( A \; B \; A^{T} \right)^{-1}  \; A \; \SMATF^{T}
$,
take $ \MATRIXLAM_{0} \in \REALS^{ \DIMG \times \DIMA } $, as
$ 
\MATRIXLAM_{0} = \SMATF \; A^{T} \; \left( A \; B \; A^{T} \right)^{-1}
$, so that by the definition of $ \MATRIXM ( \MATRIXLAM )$ it results
$ \MATRIXM ( \MATRIXLAM_{0} ) = 
\int \left( \GAMMAVEC -  \MATRIXLAM_{0} \; A \; \RHOVEC \right) \left( \GAMMAVEC -  \MATRIXLAM_{0} \; A \; \RHOVEC \right)^{T} \; \DIFF \MU
=
\int \GAMMAVEC \; \GAMMAVEC^{T} \; \DIFF \MU
-
\SMATF \; A^{T} \; \left( A \; B \; A^{T} \right)^{-1}  \; A \; \SMATF^{T}
=
0
$. Hence $ \TRACE \left( \MATRIXM ( \MATRIXLAM_{0} ) \right) = 
\TRACE \left( 
\displaystyle\int \left( \GAMMAVEC -  \MATRIXLAM_{0} \; A \; \RHOVEC \right) \left( \GAMMAVEC -  \MATRIXLAM_{0} \; A \; \RHOVEC \right)^{T} \; \DIFF \MU
\right) =$\break
$
\sum_{i = 1}^{\DIMG}
\displaystyle\int \left[ \GAMMAVEC -  \MATRIXLAM_{0} \; A \; \RHOVEC \right]_{i}^{2} \; \DIFF \MU
= 0
$,
and  then
$
\GAMMAVEC =  
\MATRIXLAM_{0} \; A \; \RHOVEC
$ $\MU$-ae.
\end{proof}
The following definition specifies all the elements required in the proposed linear matrix inequality (LMI) generalized Barankin bound.
\begin{DEFINITION} \label{def:CONDA}
Given arbitrary $ \DIMM \in \NAT$ and $ \DIMA \in \NAT$, an arbitrary real matrix $A$
of dimensions $ \DIMA \times \DIMM$, $A \in \REALS^{ \DIMA \times \DIMM }$,
and arbitrary indexes $ \INDX_{i} \in \SINDX $, for $ 1 \le i \le \DIMM$, 
define $ \TAO^{T} = ( \INDX_{1}, \INDX_{2}, \ldots, \INDX_{\DIMM})  $, $ \TAO \in {\SINDX}^{\DIMM} $,
and define the quad-tuple $ \QUADQ $, as $  \QUADQ = \left( \DIMM, \DIMA, A, \TAO \right) $.
Define $ \HT = \GT - \GTRUE $.
\\
\noindent
Define $ \VECB^{T}(\TAO) = ( \PI (\INDX_{1}), \PI(\INDX_{2}), \ldots, \PI(\INDX_{\DIMM}))  $,
i.e. $ \VECB(\TAO) \in {\THREEBETA}^{\DIMM} $,
define the $ \DIMPAR \times \DIMM$ real matrix $ G(\TAO)$ as 
$ G(\TAO) = \Bigl( \FUNCG(\INDX_{1}) - \GTRUE, \; \FUNCG(\INDX_{2}) - \GTRUE, \ldots, \; \FUNCG(\INDX_{\DIMM}) - \GTRUE  \Bigr) =
\Bigl( \FUNCH(\INDX_{1}), \; \FUNCH(\INDX_{2}), \ldots, \; \FUNCH(\INDX_{\DIMM}) \Bigr)$, 
and define the $ \DIMM \times \DIMM$ real matrix $ B(\TAO) $ as $ B(\TAO) = \EXP \left[ \VECB(\TAO) \; \VECB^{T}(\TAO) \right]$.
\\
\noindent
Define $ \CONDA $ as the collection of all the quad-tuples 
$ \QUADQ $ with $ \DET \left( A \; B(\TAO) \; A^{T}) \right) \neq 0  $., i.e.
\begin{align}
\nonumber
 \CONDA = \left\{ \QUADQ : \forall  \; \DIMM \in \NAT, \forall  \; \DIMA \in \NAT, \forall \; A \in \REALS^{ \DIMA \times \DIMM }, 
\right.
\\
\nonumber
\left.
\forall \; \TAO \in {\SINDX}^{\DIMM} , \; 
{\rm with} \; \DET \left( A \; B(\TAO) \; A^{T}) \right)  \neq 0   \right\}
\end{align}.
\end{DEFINITION}

\begin{DEFINITION} \label{def:COLLBOUNDA}
Call $ \UG$ the family of all the finite covariance at $ \INDXTRUE $ unbiased estimators of $\GT$, for all $\INDX \in \SINDX$,
for Problem \ref{prob:MAIN}.
Define $ \WA $ as the collection  of matrices of the form:
\begin{displaymath}
W ( \QUADQ ) = G(\TAO) \; A^{T} \; \left(A \; B(\TAO) \; A^{T}  \right)^{-1} \; A \; G^{T}(\TAO) \qquad \forall \QUADQ \in \CONDA 
\end{displaymath}
i.e. $\forall \DIMM \in \NAT$,  $ \forall \DIMA \in \NAT$, 
$ \forall A \in \REALS^{\DIMA \times \DIMM}$, $ \forall \TAO \in {\SINDX}^{\DIMM} $, 
with $ \DET \left( A \; B(\TAO) \; A^{T}) \right) \neq 0 $, with $G(\TAO)$ and  $B(\TAO)$ as in Definition \ref{def:CONDA}.
Hence $ \WA = \left\{ W ( \QUADQ ) : \QUADQ \in \CONDA \right\} $.
The matrices $ W ( \QUADQ ) $ will be called the Barankin covariance lower bound matrices for Problem \ref{prob:MAIN}.
\\
\noindent
Let $ S ( \THREEBETA ) $ be the linear span of $ \THREEBETA $, i.e.
$ S(\THREEBETA) = \{ u \in \SPACELTWOTRUE : u = \sum_{i=1}^{\DIMM} \ASUBI \; \PISUBI \; \WP1, 
\; \forall{\DIMM}\in \NAT, \forall \; \ASUBI \in \REALS \; {\rm for} \; 1 \le i \le \DIMM, \; \forall \; \PISUBI \in \THREEBETA \; {\rm for} \; 1 \le i \le{\DIMM}\} $.
\end{DEFINITION}
The following theorem gives the first half of the Barankin vector bound.
\begin{THEOREM} \label{theo:ONE}
If for Problem \ref{prob:MAIN} there exists a finite covariance at $ \INDXTRUE $ unbiased estimator $\PSI(\SAMP) \in \UG $  for $\GT$, $\forall \INDX \in \SINDX$, then, 
see Definition \ref{def:CONDA},
\begin{equation} \label{eq:THEONE}
\COVTRUE ( \PSI )  \ge G (\TAO)  A^{T} \left( A  B (\TAO)  A^{T} \right)^{-1}  A  G^{T}(\TAO) \; \; \forall \QUADQ \in \CONDA
\end{equation}
i.e. (\ref{eq:THEONE}) is true for the set of conditions $\CONDA$: $\forall \; \DIMM \in \NAT$,  $ \forall \DIMA \; \in \NAT$, 
$ \forall \; A \in \REALS^{\DIMA \times \DIMM}$, $ \forall \; \TAO \in {\SINDX}^{\DIMM} $, with $ \DET \left( A \; B(\TAO) \; A^{T} \right) \neq 0 $.
There is equality in (\ref{eq:THEONE}) for some $ \PSI^{*} \in \UG $ and some $ \QUADQ^{*} = \left( \DIMM^{*}, \DIMA^{*}, A^{*}, \TAO^{*} \right) $,
$ \QUADQ^{*} \in \CONDA $, if and only if
there 
exists a matrix $ \MATRIXLAM^{*} \in \REALS^{ \DIMPAR \times \DIMA } $ such that
$ \FI^{*} = \PSI^{*} - \GTRUE =  \MATRIXLAM^{*} \; A^{*} \; \VECB ( \TAO^{*} ) $ w.p. 1, if and only 
if each component $ [\FI^{*}]_{i} $ is a linear combination of elements 
in $ \THREEBETA $ w.p. 1 for $ 1 \le i \le  \DIMPAR $, i.e. $ \FI^{*} = \PSI^{*} - \GTRUE \in \bigr( S ( \THREEBETA ) \bigl)^{\DIMPAR} $,
see Definition \ref{def:COLLBOUNDA}.
\end{THEOREM}
\begin{proof}
The proof will follow from Lemma \ref{lemma:SCHUR}. Let $ \PSI \in \UG $ be an arbitrary finite covariance at $ \INDXTRUE $ unbiased estimator for
Problem \ref{prob:MAIN}. 
Take an arbitrary $ \DIMM \in \NAT $, 
and an arbitrary $ \TAO \in \SINDX^{\DIMM}$, see Definition \ref{def:CONDA}.
Since  $ \PSI $ is unbiased, see (\ref{eq:UNBIASED}),
\begin{align}
\nonumber
\EXPTRUE & \left[ \: \FI \; \VECB^{T}(\TAO)   \right] 
=  
\Bigl(  \EXPTRUE \left[ \: \FI \; \PI( \SINDX_{1} ) \right], \ldots, \EXPTRUE \left[ \; \FI \; \PI( \SINDX_{\DIMM} ) \right] \Bigr) 
\\
\nonumber
& = 
\Bigl(   \int \FI \; \PI( \SINDX_{1} ) \; \DIFF \RPTRUE, \ldots, \int \FI \;  \PI( \SINDX_{\DIMM} ) \; \DIFF \RPTRUE   \Bigr) 
\\   
\nonumber
& =
\left(   \int \FI \; \DIFF \RPTONE, \int \FI \; \DIFF \RPTTWO, \ldots,
\int \FI \; \DIFF \RPBB_{\INDX_{\DIMM}}   \right) 
\\   
\nonumber
& =
\Bigl( \FUNCH(\INDX_{1}), \FUNCH(\INDX_{2}), \ldots, \FUNCH(\INDX_{\DIMM})  \Bigr)
= G(\TAO)
\end{align}
then,
$ G(\TAO) = \EXPTRUE \left[  \FI \; \VECB^{T}(\TAO)   \right] = \EXPTRUE \left[ ( \PSI - \GTRUE ) \; \VECB^{T}(\TAO)   \right]   $, 
see Definition \ref{def:CONDA},
and this is true for any unbiased estimator $ \PSI \in \UG $.
Additionally, we have,\break
$ \int A \; \VECB(\TAO) \; \VECB^{T}(\TAO) \; A^{T} \; \DIFF \RPTRUE  = A \; B(\TAO) \; A^{T} $, 
see Definition \ref{def:CONDA}. 
Take an arbitrary $ \DIMA \in \NAT $ and a matrix $ A \in \REALS^{ \DIMA \times \DIMM } $ such that $ \DET ( A \; B(\TAO) A^{T} ) \neq 0 $
otherwise arbitrary.
Then the result follows from Lemma \ref{lemma:SCHUR} with $ \GAMMAVEC = \FI $, $ \RHOVEC = \VECB ( \TAO ) $, $ \SMATF = G(\TAO) $, and $ B = B(\TAO) $.
The first if and only if equality condition follows directly from Lemma \ref{lemma:SCHUR}.
As for the second equality condition, if there is equality in (\ref{eq:THEONE})
for some $ \PSI^{*} \in \UG $ and some $ \QUADQ^{*} = ( \DIMM^{*}, \DIMA^{*}, A^{*}, \TAO^{*} ) \in \CONDA $,
then from Lemma \ref{lemma:SCHUR}, there  exists
$ \MATRIXLAM^{*} \in \REALS^{ \DIMPAR \times \DIMA^{*} } $ such that
$ \FI^{*} = \PSI^{*} - \GTRUE =  \MATRIXLAM^{*} \; A^{*} \; \VECB ( \TAO^{*} ) $ w.p. 1. Since $ \VECB ( \TAO^{*} ) \in \THREEBETA^{ \DIMM^{*} } $,
then each component $ \left[ \FI^{*} \right]_{i} $ is a linear combination w.p. 1 of elements in $ \THREEBETA $, for $ 1 \le i \le  \DIMPAR $,
i.e. $ \FI^{*} = \PSI^{*} - \GTRUE \in \bigr( S ( \THREEBETA ) \bigl)^{\DIMPAR} $.
Conversely, suppose that $ \PSI^{*} \in \UG $, with $ \FI^{*} = \PSI^{*} - \GTRUE $, and that
each component $ \left[ \: \FI^{*} \right]_{i} $ is a linear combination w.p. 1 of elements in $ \THREEBETA$, i.e. $ \FI^{*} = \PSI^{*} - \GTRUE \in \bigr( S ( \THREEBETA ) \bigl)^{\DIMPAR} $.
Since   each $ \left[ \: \FI^{*} \right]_{i} \in S ( \THREEBETA ) $ w.p. 1, then, there exist $ M_{i} \in \NAT $, $ \AVEC_{i} \in \REALS^{  M_{i} } $, and
$ \TAO_{i} \in \SINDX^{M_{i}} $, such that $ \left[ \: \FI^{*} \right]_{i} = \AVEC_{i}^{T} \VECB ( \TAO_{i} ) $ w.p. 1 for $ 1 \le i \le \DIMPAR $.
Define $ M_{\FLOP} = \sum_{i=1}^{\DIMPAR} M_{i} $, and $ \TAO^{T}_{\FLOP} = \left( \TAO_{1}^{T}, \TAO_{2}^{T}, \ldots, \TAO_{\DIMPAR}^{T} \right) $,
$ \TAO_{\FLOP} \in {\SINDX}^{M_{\FLOP}} $. 
Call $ \VECB_{\FLOP} = \VECB ( \TAO_{\FLOP} ) $, $ \VECB_{\FLOP} \in {\THREEBETA}^{M_{\FLOP}} $.
Define the real matrix $ A_{\FLOP} \in \REALS^{ \DIMPAR \times M_{\FLOP}} $, as the block-diagonal matrix
$  A_{\FLOP} = \DIAG \left( \AVEC_{1}^{T}, \AVEC_{2}^{T}, \ldots, \AVEC_{\DIMPAR}^{T} \right) $, 
where each block $ \AVEC_{i}^{T} $ is of dimension $ 1 \times \ M_{i}$, for $ 1 \le i \le M_{i} $,
so that $ \FI^{*}  =  A_{\FLOP} \;  \VECB_{\FLOP} $.
Starting  with the second component of $  \VECB_{\FLOP} $, see Observation \ref{obs:THREEBETA}, delete the i-th component if it is a linear combination w.p. 1 of the previous components. There will remain $ M_{\PLOP} \in \NAT $ elements,  with $ 1 \le M_{\PLOP} \le M_{\FLOP} $, 
see Observation \ref{obs:THREEBETA}. 
Call $ \TAO_{\PLOP} \in \SINDX^{ M_{\PLOP} } $ the non-deleted indexes of the previous elimination procedure. 
Call $ \VECB_{\PLOP} = \VECB ( \TAO_{\PLOP} ) $, $ \VECB_{\PLOP} \in \THREEBETA^{ M_{\PLOP} } $, so that the components of $\VECB_{\PLOP} $ are linearly independent w.p. 1. Then, there exists a real matrix $ A_{\PLOP} \in \REALS^{ M_{\FLOP} \times M_{\PLOP} } $ such that
$   \VECB_{\FLOP} = A_{\PLOP} \; \VECB_{\PLOP} $ w.p. 1, and then $ \FI^{*}  =  A_{\FLOP} \;  A_{\PLOP} \; \VECB_{\PLOP} $ w.p. 1.
Define the quad-tuple $ \QUADQ_{\PLOP} = \left( M_{\PLOP}, M_{\PLOP}, I_{\PLOP}, \TAO_{\PLOP} \right) $, where $ I_{\PLOP} $ is the identity matrix of dimensions $ M_{\PLOP} \times M_{\PLOP} $. 
Call $ B_{\PLOP} = B( \TAO_{\PLOP} ) = \EXPTRUE \left[ \VECB_{\PLOP} \; \VECB_{\PLOP}^{T}  \right] $, 
$ B_{\PLOP} \in \REALS^{  M_{\PLOP} \times M_{\PLOP} } $, so that $ \DET ( B_{\PLOP} ) \neq 0 $. If not, there would exist
$ \ALFAVEC \in \REALS^{M_{\PLOP}}$, with $ \ALFAVEC \neq 0 $, such that $ \ALFAVEC^{T} B_{\PLOP} \ALFAVEC = 0 $, but
$ \ALFAVEC^{T} B_{\PLOP} \ALFAVEC = \ALFAVEC^{T} \EXPTRUE \left[ \VECB_{\PLOP} \; \VECB_{\PLOP}^{T}  \right] \ALFAVEC = 
\EXPTRUE \left[ \ALFAVEC^{T} \VECB_{\PLOP} \; \VECB_{\PLOP}^{T}  \ALFAVEC \right] =
\EXPTRUE \Bigl[ \left(\ALFAVEC^{T} \VECB_{\PLOP} \right)^{2} \Bigr] =
\NORMLTWO { \ALFAVEC^{T} \VECB_{\PLOP} }
$, and then it would be $ \NORMLTWO { \ALFAVEC^{T} \VECB_{\PLOP} } = 0 $, which is a contradiction since the components of $ \VECB_{\PLOP} $
are linearly independent w.p. 1. Hence,
$ \DET ( I_{\PLOP} B_{\PLOP} I_{\PLOP}^{T} ) = \DET ( B_{\PLOP} ) \neq 0 $, so that $ \QUADQ_{\PLOP}  \in \CONDA $.
Since $ G ( \TAO_{\PLOP} ) = \EXPTRUE \left[ \FI^{*} \; \VECB_{\PLOP}^{T} \right] = 
\EXPTRUE \Bigl[ A_{\FLOP} \;  A_{\PLOP} \; \VECB_{\PLOP} \; \VECB_{\PLOP}^{T} \Bigr] =
A_{\FLOP} \;  A_{\PLOP} \; \EXPTRUE \Bigl( \VECB_{\PLOP} \; \VECB_{\PLOP}^{T} \Bigr) =
A_{\FLOP} \;  A_{\PLOP} \; B_{\PLOP} 
$, 
then 
$
W ( \QUADQ_{\PLOP} )
= 
G( \TAO_{\PLOP} ) \; B_{\PLOP}^{-1} \; G^{T}( \TAO_{\PLOP} ) = 
A_{\FLOP} \;  A_{\PLOP} \; B_{\PLOP} \; B_{\PLOP}^{-1} \; B_{\PLOP} \; A^{T}_{\PLOP} \; A^{T}_{\FLOP} =
A_{\FLOP} \;  A_{\PLOP} \; B_{\PLOP} \; A^{T}_{\PLOP} \; A^{T}_{\FLOP} 
$.\XBREAK 
But $ \COVTRUE ( \PSI^{*} ) = \EXPTRUE \left[ \FI^{*} \; (\FI^{*})^{T}  \right]
=
\EXPTRUE \Bigl[ A_{\FLOP} \;  A_{\PLOP} \; \VECB_{\PLOP} \left( A_{\FLOP} \;  A_{\PLOP} \; \VECB_{\PLOP} \right)^{T} \Bigr]
=\XBREAK
A_{\FLOP} \,  A_{\PLOP} \; \EXPTRUE \Bigl[ \VECB_{\PLOP} \, \VECB^{T}_{\PLOP} \Bigr]  A^{T}_{\PLOP} \, A^{T}_{\FLOP}
=
A_{\FLOP} \;  A_{\PLOP} \; B_{\PLOP} \; A^{T}_{\PLOP} \; A^{T}_{\FLOP}
$, and then $ \COVTRUE ( \PSI^{*} ) = W ( \QUADQ_{\PLOP} ) $.
\end{proof}
The converse of this theorem, 
is given in Theorem \ref{theo:MAIN}, see Section \ref{subsec:MAIN}.
\begin{OBSERVATION}   \label{obs:MATG}
The previous proof shows that if $ \PSI \in \UG $ is a finite covariance 
\ifundefined{DRAFT}
at $ \INDXTRUE $
\else
\fi
unbiased estimator for Problem \ref{prob:MAIN}, then
$
G(\TAO) = \Bigl( \FUNCH(\INDX_{1}), \FUNCH(\INDX_{2}), \cdots, \FUNCH(\INDX_{\DIMM})  \Bigr) = \EXPTRUE \left[  \FI \; \VECB^{T}(\TAO)   \right] 
$, so that the value of\break
$ \EXPTRUE \left[  \FI \; \VECB^{T}(\TAO)   \right]  $ is independent of the estimator $ \PSI \in \UG $ as a consequence of the unbiasedness of $ \PSI $, see (\ref{eq:UNBIASED}).
\end{OBSERVATION}
\begin{OBSERVATION} \label{obs:PARORD}
Theorem \ref{theo:ONE} shows that any other finite covariance at $ \INDXTRUE $ unbiased estimator will satisfy (\ref{eq:THEONE}).
Then,
the covariance matrix of any unbiased estimator in $ \UG $ is comparable, in the L\"owner partial order, with any of the matrices in $ \WA $.
Hence:
\begin{displaymath}
\COVTRUE ( \PSI ) \ge  W \qquad \forall \; \PSI \in \UG \;  {\tt and } \;  \forall \; W \in \WA
\end{displaymath}
with equality if and only if $ \FI = \PSI - \GTRUE \in \bigr( S ( \THREEBETA ) \bigl)^{\DIMPAR} $. The covariance matrices of estimators in $ \UG $ need not be comparable
between them, as well as, Barankin bound matrices in $ \WA $ need not be comparable between them. 
\end{OBSERVATION}
\section{Functional analysis setup} \label{sec:FAS}
\subsection{Definition of the operator $ \OPZ : \THREEBETA \to \RPAR $ } \label{sub:BETA-OP}
From Hypothesis 
\ref{hypo:TWO} we have $\THREEBETA \subseteq \SPACELTWOTRUE$.
The subset $\THREEBETA$ is not a linear subspace, since any $ \PI \in \THREEBETA $, is a Radon-Nykodim derivative of a p.m. 
with respect to the p.m. $ \RPTRUE $, then $ \PI \ge 0 $ w.p. 1, \cite{HEWITT}, p. 315, 
with $ \NORMLTWO { \PI } \neq 0 $, see Observation \ref{obs:THREEBETA},
so that  $ - \PI $ cannot belong to $ \THREEBETA $.

Let $ u_{0} $ be an arbitrary element in $\THREEBETA$.  To this particular element $ u_{0} \in \THREEBETA$ corresponds a unique $\INDX_{0} \in \SINDX$, such that $ u_{0} \equiv \PI(\INDX_{0})$, see Hypothesis \ref{hypo:BASIC}, 
so that $ \INDX_{0} = \PI^{-1}( u_{0})$,
and,  to this index $\INDX_{0}$ corresponds a unique well defined value $ \HTZERO = \GTZERO - \GTRUE \in \RPAR$. 
Hence, to $ u_{0} \in \THREEBETA $ corresponds a unique element $ \FUNCH \bigl( \PI^{-1}( u_{0} ) \bigr) \in \RPAR $ 
which we define as $ \OPZ \bigl( u_{0}  \bigr) $,
so that $ \OPZ ( u_{0} ) =  \FUNCH \bigl( \PI^{-1}( u_{0} ) \bigr) $. 
Hence,
\begin{equation} \label{eq:OPZ}
\OPZ ( \PI ( \INDX ) ) =  \HT  \qquad \forall \INDX \in \SINDX 
\end{equation}
equivalently $ \OPZ ( u ) = \FUNCH \bigl( \PI^{-1} ( u ) \bigr)  $,
for all $ u \in \THREEBETA $.
Note that $ \OPZ ( \PI(\INDXTRUE) ) =  \HTRUE = 0 $. 
Then, we may establish a direct relation from  $\THREEBETA$ to $\RPAR$,
as an operator $\OPZ$ from $\THREEBETA$ to $\RPAR$, i.e.
$\OPZ : \THREEBETA \to \RPAR$. This operator is not (without additional conditions) necessarily linear nor bounded. 
The operator $ \OPZ $ is completely defined by the collection of Radon-Nykodim derivatives in $\THREEBETA$, i.e.
the elements $ \PIT \in \THREEBETA $, for all $ \INDX \in \SINDX $, and the vectors $ \GT \in \RPAR $, for all $ \INDX \in \SINDX $,
and does not depend on the existence or not of any unbiased estimator, and if it exists,
on whether it has finite covariance at $ \INDXTRUE $ or not.
\subsection{Barankin formulation}
The key observation made by Barankin, \cite{BARANKIN}, for $ \DIMPAR = 1 $, where he considers $ \THREEBETA \subseteq \SPACELQ $, $ \DIMQ \ge 1 $, is that if we are able to find an integral representation of the operator $ \OPZ $, then the problem is solved.

In Barankin, \cite{BARANKIN}, the answer is given by the Riesz Representation Theorem which finds an element in the conjugate space $ \phi_{0} \in \SPACELP $, with $ 1/s + 1/r = 1 $,
such that $ \OPZ ( u ) = \int \phi_{0}  \; u \; \DIFF \RPTRUE $, $ \forall u \in \THREEBETA $, with minimum $\DIMP$-norm, i.e. minimum
$\DIMP$-th variance. In our case, we generalize to vector estimates, i.e. $ \DIMPAR > 1 $, but we will only consider the case $ \DIMP = \DIMQ = 2$ which is the traditional variance and covariance matrices case, which is the most important in applications. To solve the problem the idea is to generalize the Riesz representation theorem to the vector case. The Riesz representation theorem requires that the represented functional be defined from a linear
space to the reals. Since $ \THREEBETA $ is not a linear subspace, Barankin, see \cite{BANACH}, pp. 479-480, extends the operator $ \OPZ $ to a linear operator over the whole space, using indirectly the Hahn-Banach theorem, invoking a condition first used by Riesz and generalized by Helly as exposed in \cite{BANACH} footnote in p. 56, see also \cite{NARICI}. In the next sub-section we generalize the Helly-Riesz-Banach condition to handle the vector case.
In Section \ref{sec:GRRT} we generalize the Riesz representation theorem to the vector case 
without requiring the Hahn-Banach theorem, 
and in Section \ref{sec:OPTEST} we apply these results to solve Problem \ref{prob:MAIN}.
\subsection{Vector generalized Barankin hypothesis: Helly, Riesz, Banach,\ZBREAK Barankin (HRBB)}  \label{subsec:HRBB}
The following is the generalization of the hypothesis in \cite{BARANKIN},
pp. 480 and 483--484, see also \cite{BANACH}, Theorems 4 and 5 pp. 55--57.
This condition will be called here the HRBB condition for Helly, Riesz, Banach, Barankin.
For $ u \in \SPACELTWOTRUE $, define the semi-norm $ \NORMLTWO { u } = \bigl( \int u^{2} \; \DIFF \RPTRUE \bigr)^{1/2} $, and call
$ \NORMRPAR { \XVEC } $ the standard Euclidean norm for $ \XVEC \in \RPAR $.
\begin{DEFINITION}{(HRBB condition)} \label{def:HRBB}
The functions $\HT = \GT - \GTRUE$, $\HT \in \RPAR $, and $\PIT \in \THREEBETA $, $\forall \INDX \in \SINDX$, satisfying
the Hypothesis \ref{hypo:ONE}, \ref{hypo:BASIC}, and \ref{hypo:TWO}, for Problem \ref{prob:MAIN},
satisfy the HRBB condition iff:
$ \exists \KH \in \RPLUS $, i.e. $ \KH \ge 0$,
such that:
\begin{equation}  \label{eq:HRBB}
{\NORMGTI} \le \KH \; {\NORMPITI}
\end{equation}
for all $ \DIMM \in \NAT$, for all $ \ASUBI \in \REALS $, $i = 1,2 , \cdots, \DIMM $, 
for all $\TSUBI \in \SINDX$, $i = 1,2 , \cdots, \DIMM $.
\end{DEFINITION}
\section{Generalized Riesz representation theorem}   \label{sec:GRRT}
Here a generalization is given of the Riesz Representation Theorem for Hilbert spaces real functionals,
see e.g. \cite{BEALS} p. 112, to operators from 
an arbitrary Hilbert space $\HILBERT$, separable or not, to the real  finite dimensional vector space $\RPAR$, with $ \DIMPAR \ge 1 $.
The proof given here does not require the Hahn-Banach
extension theorem, and then, the non-denumerable Axiom of Choice is not required, or some less stringent variant, \cite{NARICI}. 
The bound proposed in Helly's theorem, \cite{BANACH} pp. 55--56, is generalized, and will be called the operator OP-HRBB (Helly, Riesz, Banach, Barankin) condition.
\subsection{The Theorem.}
Let $\HILBERT$ denote an arbitrary Hilbert space with semi-inner product $ \PRODH { u, v } $, $ \forall u, v \in \HILBERT $, 
and semi-norm $ \NORMH { u } = \PRODH {  u, u  }^{1/2}$. If $ \NORMH { u } = 0$, then we say that
$ u = 0 $ \emph{ in semi-norm}, (i.s.n.).
Equivalently $ u = 0 \; \ISN$ iff $ \NORMH { u } = 0$. Define $ u = v \; \ISN$, iff $ \NORMH { u - v }  = 0$.
\begin{THEOREM}  \label{theo:OP-HRBB}
Let $\HILBERT$ be an arbitrary Hilbert space. Let $\THREEBETA$ be a non-empty arbitrary subset of $\HILBERT$, $\THREEBETA \neq \emptyset $,
$\THREEBETA \subseteq \HILBERT $. Let $\OPZ$ be an operator from $\THREEBETA$ to $\RPAR$, $ \OPZ: \THREEBETA \to \RPAR$, such that 
there exists at least one $u_{0} \in \THREEBETA$ for which $ \OPZ( u_{0} ) \neq 0 $.
Assume that the operator $\OPZ$ satisfies the following condition, that will be called the operator HRBB condition (OP-HRBB): 
$ \exists \KH \in \RPLUS $, i.e. $ \KH \ge 0$,
such that:
\begin{equation} \label{eq:OP-HRBB}
\parallel \sum_{i = 1}^{\DIMMOP} \ASUBI \; \OPZ ( u_{i} ) \parallel_{\RPAR} 
\; \le \; \KH \; 
\parallel \sum_{i = 1}^{\DIMMOP} \ASUBI \; u_{i} \parallel_{\HILBERT} 
\end{equation}
for all $\DIMMOP \in \NAT$, for all $ \ASUBI \in \REALS $, $i = 1,2 , \cdots, \DIMMOP $, 
for all $ u_{i} \in \THREEBETA $, $i = 1,2 , \cdots, \DIMMOP $. 

Call $ C( \THREEBETA ) \subseteq \HILBERT $ the minimal closed linear space containing $\THREEBETA$, 
i.e. the closed linear span of $ \THREEBETA $, \cite{CONWAY} p. 11.
Then:
\begin{enumerate}
\item 
The operator $\OPZ$ may be extended to a bounded linear operator $\OPC$ from $ C(\THREEBETA) $ to $ \RPAR $,
 $\OPC: C(\THREEBETA) \to \RPAR$, with
$ \OPC (u) = \OPZ(u) $, for all $ u \in \THREEBETA$.
\item
The operator $\OPC$ has the following representation:
There exists $ \DIML \in \NAT$, $ 1 \le \DIML \le \DIMPAR$, 
and there exist orthonormal $ \UHI $'s, $ \UHI \in C( \THREEBETA ) $, for $ 1 \le i \le \DIML$,
such that:
\begin{equation}   \label{eq:REP}
\OPC ( u ) = \sum_{i=1}^{\DIML} \; \PRODH { u, \UHI } \OPC ( \UHI ) \qquad \forall u \in C(\THREEBETA)
\end{equation}
\end{enumerate}
\end{THEOREM}
\begin{OBSERVATION}
The standard Riesz representation theorem, corresponds to $ \THREEBETA = S ( \THREEBETA ) = C ( \THREEBETA ) = \HILBERT $, and $ \DIMPAR = 1$.
In that case the operator $ \OPZ $ is taken as a bounded linear operator, so that the OP-HRBB condition is satisfied, and then the conclusion is given
by (\ref{eq:REP}) with $ \DIML = 1 $.
\end{OBSERVATION}
\subsection{Proof of the generalized Riesz representation theorem}
\subsubsection{Extension of the operator $ \OPZ $ to the span of $\THREEBETA $, 
$\OPS: S(\THREEBETA) \to \RPAR $}        \label{subsection:LS-EXTENSION}
This extension follows the exposition of Banach in \cite{BANACH} pp. 55--56.  Assume the OP-HRBB condition is true.
Call $S(\THREEBETA) $ the linear span i.s.n. of $\THREEBETA$, i.e. 
$ S(\THREEBETA) = \{ u \in \HILBERT : u = \sum_{i=1}^{\DIMMOP} \ASUBI \; \PISUBI \; \ISN, \; \forall{\DIMMOP}\in \NAT, \forall \; \ASUBI \in \REALS, \; {\rm for} \; 1 \le i \le \DIMMOP, \; \forall \; \PISUBI \in \THREEBETA, \; {\rm for} \; 1 \le i \le{\DIMMOP}\} $.
The span $S(\THREEBETA)$ is called $[{\BETABK_{0}}]$ in \cite{BARANKIN} p. 495. Clearly, $S(\THREEBETA)$ is a linear space.
With the help of the OP-HRBB condition extend the operator $\OPZ:\THREEBETA \to \RPAR$ to an operator
$\OPS:S(\THREEBETA) \to \RPAR$ by the following procedure: for each $ u \in S(\THREEBETA)$ there exist, dependent on each $u$, $ \DIMMOP \in \NAT$, $\ASUBI$'s $\in \REALS$, $ 1 \le i \le \DIMMOP $,
$\PISUBI$'s $\in \THREEBETA $, $ 1 \le i \le \DIMMOP $, such that $ u = \sum_{i=1}^{\DIMMOP} \ASUBI \; \PISUBI \; \ISN $
Define $\OPS (u)$ for $u \in S(\THREEBETA)$ as $ \OPS (u) = \sum_{i=1}^{\DIMMOP} \ASUBI \; \LPISUBI $. 
This procedure gives a well defined value for $\OPS(u)$, since for any other decomposition of 
$u = \sum_{j = 1}^{\DIMMOP^{\prime}} \APJ \; \PIPJ  \; \ISN $, 
resulting in $ \OPS^{\prime} (u)  = \sum_{j = 1}^{\DIMMOP^{\prime}} \APJ  \LPIPJ $,
because of the OP-HRBB condition we will have:
\begin{displaymath}
\parallel \OPS (u) -  \OPS^{\prime} (u) \parallel_{\RPAR}
 \le 
\KH \parallel \sum_{i=1}^{\DIMMOP} \ASUBI  \PISUBI - \sum_{j = 1}^{\DIMMOP^{\prime}} \APJ  \PIPJ \parallel_{\HILBERT} = 0
\end{displaymath}
so that $ \sum_{i=1}^{\DIMMOP} \ASUBI \; \LPISUBI  = \sum_{j = 1}^{\DIMMOP^{\prime}} \APJ \; \LPIPJ $. The important result here is that now $S(\THREEBETA)$, unlike $\THREEBETA$, is a linear space, and that $\OPS:S(\THREEBETA) \to \RPAR$ is a bounded linear operator with bound $\KH$, i.e. 
$ \parallel \OPS (u) \parallel_{\RPAR} \; \le \; \KH \parallel u \parallel_{\HILBERT}$, $\forall u \in S(\THREEBETA)$,
and $ \OPS (u) = \OPZ(u)$, $\forall u \in \THREEBETA$. 
\begin{OBSERVATION}
Barankin, \cite{BARANKIN} pp. 480 and 483--484, following \cite{BANACH} Theorems 2 and 4, p. 55, invokes the Hahn-Banach theorem,
see e.g. \cite{CONWAY} p. 78 or \cite{BANACH} Theorem 1 p. 27, to extend the operator $\OPS$ to the whole space. The Hahn-Banach theorem requires the Axiom of Choice or some slightly less stringent condition, see e.g. \cite{NARICI}.
In \cite{BANACH} arbitrary Banach spaces are considered. The fact that here we work with Hilbert spaces, permits us to avoid the use of the Hahn-Banach theorem, and then, the non-denumerable Axiom of Choice is not required. 
\end{OBSERVATION}
\subsubsection{Extension of the operator $ \OPS $ to the closure of 
the span of $\THREEBETA $, $\OPC: C(\THREEBETA) \to \RPAR $} \label{subsub:EXTC}
Define the closure of the span of $\THREEBETA $  as $ C(\THREEBETA) = {\tt Closure}(S(\THREEBETA )) $, i.e.
$ C(\THREEBETA) =  \{ u \in \HILBERT: \exists \; (u_{n})_{n \in \NAT} \; {\rm with} \; u_{n} \in S(\THREEBETA) \; \forall n \in \NAT, \;
{\rm such \; that} \; \parallel u_{n} - u \parallel_{\HILBERT} \: \to 0 \} $.
It is readily checked that $ C(\THREEBETA) $ is a closed linear subspace of $\HILBERT$.
The set $ C(\THREEBETA) $ is called $ \{{\BETABK_{0}}\} $ in \cite{BARANKIN}, p. 494.
Extend the operator $\OPS:S(\THREEBETA) \to \RPAR$  to an operator
$\OPC:C(\THREEBETA) \to \RPAR$ by continuity:
Let $ u \in C(\THREEBETA)$, then there exists a sequence $ (u_{n})_{n \in \NAT}$ of elements $ u_{n} \in S(\THREEBETA)$
such that $ {\parallel u_{n} - u \parallel_{\HILBERT} \to 0} $. Hence this sequence is a Cauchy fundamental sequence, i.e. for each $\epsilon > 0$ there exists $N(\epsilon)$ such that $ \forall \; n,m \ge N(\epsilon)$
we have $ \parallel u_{n} - u_{m} \parallel_{\HILBERT} \; < \; \epsilon $. But since $\OPS$ is a bounded linear operator, then
$ \parallel \OPS(u_{n}) - \OPS(u_{m}) \parallel_{\RPAR} \; \le \; {\KH \; \parallel u_{n} - u_{m} \parallel_{\HILBERT}}
\; < \; \KH \; \epsilon $. 
Then, $ (\OPS(u_{n}))_{n \in \NAT}$ is a Cauchy fundamental sequence in the complete finite dimensional vector space $\RPAR$, 
\cite{BEALS} p. 23, hence there exists a limit in $\RPAR$. Call that limit $\OPC(u)$, so that
$ \parallel \OPS(u_{n}) - \OPC(u) \parallel_{\RPAR} \to 0 $, and then $ \OPS(u_{n}) - \OPC(u) \to 0 $
component by component (c.b.c.), i.e. $ [\OPS(u_{n})]_{i} - [\OPC(u)]_{i} \to 0 $, for $ 1 \le i \le \DIMPAR$.
The value $\OPC(u)$ is well defined: assume that for some other sequence 
$ (u^{\prime}_{j})_{j \in \NAT}$ of elements $u^{\prime}_{j} \in S(\THREEBETA)$ with $ \parallel  u^{\prime}_{j} - u \parallel_{\HILBERT}
\; \to \; 0$, we obtain using the previous procedure a limit $\OPC^{\prime}(u)$ for the sequence 
$ \left( \OPS(u^{\prime}_{j}) \right)_{j \in \NAT}$, i.e. $ \parallel \OPS(u^{\prime}_{j}) - \OPC^{\prime}(u) \parallel_{\RPAR} \to 0 $. 
We have:
$ \parallel  \OPS(u^{\prime}_{j}) - \OPS(u_{n}) \parallel_{\RPAR} \; =\break
\parallel \OPS(u^{\prime}_{j} - u_{n}) \parallel_{\RPAR}$
$
{ \; \le \; \KH \parallel u^{\prime}_{j} - u_{n} \parallel_{\HILBERT}  }
$
$
\; = \; \KH \parallel \; u^{\prime}_{j} \; - \; u - \; (u_{n} - u) \parallel_{\HILBERT} 
\; \le \; \KH \; ( \parallel u^{\prime}_{j} - u \parallel_{\HILBERT} + \parallel u_{n} - u \parallel_{\HILBERT})
 $. Then, $ \parallel \OPC(u) - \OPC^{\prime}(u) \parallel_{\RPAR}
\; = \;  \parallel \OPC(u) - \OPS(u_{n}) - ( \OPC^{\prime}(u) - \OPS(u^{\prime}_{j}) )
$
$
- { ( \OPS(u^{\prime}_{j}) - \OPS(u_{n}) ) \parallel_{\RPAR} } 
$
$
\; \le\break
\parallel \OPC(u) - \OPS(u_{n}) \parallel_{\RPAR}
\; + \; { \parallel \OPC^{\prime}(u) - \OPS(u^{\prime}_{j})  \parallel_{\RPAR} }
\; + \; \KH \; \Bigl( \parallel u^{\prime}_{j} - u \parallel_{\HILBERT} +\break
\parallel u_{n} - u \parallel_{\HILBERT} \Bigr)
$, so that taking the limits $ n \to \infty$, and $ j \to \infty$,
we obtain $ \OPC(u) = \OPC^{\prime}(u)$,
so that the value $ \OPC(u) \in \RPAR$ is independent of the chosen sequence.
Hence $\OPC(u)$ is a well defined operator from the closed linear subspace 
$ C(\THREEBETA) \subseteq \HILBERT$ to $\RPAR$. 
It is immediate to show that this operator is linear and that
$ \OPC(u) = \OPS (u)$, $\forall u \in S(\THREEBETA)$, and then $ \OPC (u) = \OPS (u) = \OPZ(u) $, $\forall u \in \THREEBETA$.
Finally, let's show that the operator $\OPC$ is bounded with bound $\KH$. 
Let $ u \in C(\THREEBETA)$, and $ (u_{n})_{n \in \NAT}$ a sequence of elements $ u_{n} \in S(\THREEBETA)$ such that 
$\parallel u - u_{n} \parallel_{\HILBERT} \to 0 $, and then $ \NORMRPAR { \OPC(u) - \OPS(u_{n}) } \to 0 $.
Since $ { \big|  \NORMH { u } - \; \NORMH { u_{n} } \big| }  \; \le \; 
\NORMH { u - u_{n} } $, then $\NORMH { u_{n} } \; \to \; \NORMH { u } $.
Hence, $ \NORMRPAR { \OPC (u) } = \NORMRPAR { \OPC(u) - \OPS(u_{n}) + \OPS(u_{n}) } \; \le \;
\NORMRPAR { \OPC(u) - \OPS(u_{n}) } \; + \; \NORMRPAR { \OPS(u_{n}) }
\; \le \; \NORMRPAR { \OPC(u) - \OPS(u_{n}) } \; + \; \KH \;  \NORMH { u_{n} } $. 
Taking the limit $ n \to \infty $, we obtain  $ \parallel \OPC (u) \parallel_{\RPAR} \le \KH \parallel u \parallel_{\HILBERT} $.
Hence $ \OPC (u) $ is a bounded linear operator from $ C(\THREEBETA)$ to $\RPAR$, such that
$ \OPC ( u ) = \OPS ( u ) $, $ \forall u \in S ( \THREEBETA ) $, and $ \OPC ( u ) = \OPS ( u ) = \OPZ ( u ) $, $ \forall u \in \THREEBETA $
\subsubsection{Null space $\KERC$  and 
topological complement $\KERPER$ of the operator $\OPC: C(\THREEBETA) \to \RPAR $   } \label{subsub:KERNEL}
Define the kernel or null space of the operator $\OPC$ as $ \KERC = \{ u \in C(\THREEBETA): \OPC(u) = 0 \}$.
It is readily seen that $\KERC$ is a closed linear subspace of $C(\THREEBETA)$, $ \KERC \subseteq C(\THREEBETA) \subseteq  \HILBERT $.
The orthogonal complement of $\KERC$ with respect to $C(\THREEBETA)$ is 
$\KERPER = \{ u \in C(\THREEBETA): \; \PRODH { u,w } = \; 0 \; \; \forall w \in \KERC \} $. Note that the orthogonal complement of $\KERPER$ with
respect to $C(\THREEBETA)$ is $\KERC$. It is readily shown that $\KERPER$ is a closed linear subspace of $C(\THREEBETA)$,
$ \KERPER \subseteq C(\THREEBETA) \subseteq \HILBERT $.
Next, let's show that $ C(\THREEBETA) = \KERPER \oplus \KERC $, i.e. for each $u \in C(\THREEBETA)$ 
there exist unique elements i.s.n. $ v \in \KERPER$ and $ w \in \KERC$, such that $ u = v + w \; \ISN$
We have:

{\bf Fact 1}, (Minimum Distance to a Convex Set, \cite{CONWAY} p. 8)
Let $ u \in C(\THREEBETA)$, since $\KERC$ is a closed convex subset of the {\em complete} Hilbert vector space $\HILBERT$, there exists a $ w(u) \in \KERC$, such that $ \parallel u - w(u) \parallel_{\; \HILBERT} \; \le \;
\parallel u - z \parallel_{\; \HILBERT}$, $\forall z \in \KERC$, and that element is unique i.s.n., i.e. 
if there exists another $ w^{\prime}(u) \in \KERC $
such that $ \parallel u - w^{\prime}(u) \parallel_{\; \HILBERT} \; \le \;
\parallel u - z \parallel_{\; \HILBERT}$, $\forall z \in \KERC$, then  $ \parallel w(u) - w^{\prime}(u) \parallel_{\; \HILBERT} \; = \; 0 $.

{\bf Fact 2}, (Principle of Orthogonality, \cite{CONWAY} p. 9)
Define $ v(u) = u - w(u)$, then $ v ( u ) $ is  orthogonal to each of the elements
in $\KERC$, so that $ v(u) \in \KERPER $. Additionally, if Fact 2 is true then Fact 1 is true.
The element $  w(u) $ is defined as the orthogonal projection of $ u $ on the closed subspace $ \KERC $ denoted as 
$ w(u) = \PROJ ( u \mid \KERC ) $, similarly $ v(u) = \PROJ ( u \mid \KERPER )$.

Hence $u \in C(\THREEBETA) $ may be decomposed  as $ u = v(u) + w(u) \; \ISN $ with $ v(u) \in \KERPER$ and
$ w(u) \in \KERC $. This decomposition is unique i.s.n.: if we also may write $ u = v^{\prime}(u) + w^{\prime}(u) \; \ISN $, 
with $ v^{\prime}(u) \in \KERPER$ and $ w^{\prime}(u) \in \KERC$, 
then $ v(u) - v^{\prime}(u) = w^{\prime}(u) -w(u) \; \ISN $ with 
$ v(u) - v^{\prime}(u) \in \KERPER$ and $ w^{\prime}(u) -w(u) \in \KERC $ by linearity.
Then, 
$ \parallel v(u) - v^{\prime}(u) \parallel_{\HILBERT}^{2} \; = \; { \PRODH { v(u) - v^{\prime}(u), v(u) - v^{\prime}(u) } }
\; = \; \PRODH { v(u) - v^{\prime}(u), w^{\prime}(u) -w(u) } \; = \; 0 $.
Similarly $ \parallel w(u) - w^{\prime}(u) \parallel_{\HILBERT}^{2} = 0 $. Hence $\KERC$ and $\KERPER$ are topological complements, 
\cite{CONWAY} p. 93, i.e.  $ C(\THREEBETA) = \KERPER \oplus  \KERC $.
\subsubsection{ Images  of $\THREEBETA$, $S(\THREEBETA)$, $ C(\THREEBETA) $ and $ \KERPER $ } \label{subsub:IMAGES}
The previous properties are valid if we replace 
the space $ \RPAR $
with an arbitrary Banach space. The following properties depend strongly on the finite dimensional character
of $ \RPAR $.
The main property is that $\KERPER$ is a finite dimensional sub-space of $\HILBERT$ as shown below.

Call $ \IMAGE [ \THREEBETA ]$ the image of the operator $ \OPZ: \THREEBETA \to \RPAR $, then, 
$ \IMAGE [ \THREEBETA ] \subseteq \RPAR$. 
Since $\RPAR$ has
dimension $\DIMPAR$ then any $\DIMPAR + 1$ vectors in $\RPAR$ are linearly dependent, and  there are $\DIMPAR$ linearly independent vectors that constitute a basis for $\RPAR$, see e.g. \cite{BIRKHOFF} pp. 178-179.
Since $ \IMAGE [ \THREEBETA ] \subseteq \RPAR$ then any $\DIMPAR + 1$ vectors in  $ \IMAGE [ \THREEBETA ]$ are linearly dependent. 
Since by hypothesis 
there exists at least one $ u_{0} \in \THREEBETA$ such that $ \OPZ(u_{0}) \neq 0$, then there exists $ \DIML \in \NAT $ with $ 1 \leq \DIML \le \DIMPAR $, such that any $\DIML + 1$ vectors in 
$ \IMAGE [ \THREEBETA ]$ are linearly dependent, and there are $\DIML$ linearly independent 
vectors $\OPZ(\widehat{\PI}_{1}), \OPZ(\widehat{\PI}_{2}), \cdots, \OPZ(\widehat{\PI}_{\DIML})$ that belong to $ \IMAGE [ \THREEBETA ]$ with $ \PIHI \in \THREEBETA $, for $ 1 \le i \le \DIML $. 
Note that $ \IMAGE [ \THREEBETA ]$ is not necessarily a linear subspace.

The elements $ \PIHI \in \THREEBETA $ for $ 1 \le i \le \DIML $, are linearly independent i.s.n., 
i.e. whenever there are real coefficients $\ASUBI \in \REALS $ for $ 1 \le i \le \DIML $, for which we have $\NORMPIH = 0$, then  $\ASUBI = 0 $ for $ 1 \le i \le \DIML $.  If not, there would exist
$\ASUBI$'s,  $\ASUBI \in \REALS $ for $ 1 \le i \le \DIML $, not all null, such that $ \NORMPIH \; = \; 0 $,
but then, because of the OP-HRBB condition $ \NORMLPIH \le \KH \NORMPIH$, see (\ref{eq:OP-HRBB}), it would be  $\NORMLPIH = 0$, iff
$ \SLPIH = 0 $, but the $ \OPZ( \PIHI ) $'s are l.i., so that it should be $ \ASUBI = 0 $, $ 1 \le i \le \DIML $, which is a contradiction.

Next, decompose each $ \PIHI $ as in the previous item \ref{subsub:KERNEL}, i.e. for $ 1 \le i \le \DIML $, $ \PIHI = \VHI + \WHI \; \ISN $,
where $ \VHI = \PROJ (\PIHI \mid \KERPER ) $ and\break
$ \WHI = \PROJ (\PIHI \mid \KERC ) $, so that
$ \VHI \in \KERPER \subseteq C( \THREEBETA )$ and $ \WHI \in \KERC  \subseteq C( \THREEBETA )$.
Note that, even though $ \PIHI \in \THREEBETA $, and then $ \PIHI \in S( \THREEBETA ) $, in general it may happen that
$ \VHI \notin S( \THREEBETA ) $ and $ \WHI \notin S( \THREEBETA ) $.
Since, see item \ref{subsub:EXTC}, $ \OPZ( \PIHI ) = \OPC ( \PIHI ) = \OPC  (\VHI + \WHI ) =
\OPC ( \VHI ) + \OPC ( \WHI ) = \OPC ( \VHI )$, then the vectors  $ \OPC ( \VHI ) $'s are linearly independent. 
Hence, the elements $ \VHI $'s are linearly independent i.s.n.: if not, there would exist
$\ASUBI$'s,  $\ASUBI \in \REALS $ for $ 1 \le i \le \DIML $, not all null, such that $ \NORMVH \; = \; 0 $. 
Then, since the extension $ \OPC $ is a bounded linear operator, see item \ref{subsub:EXTC}, then $ \NORMLVH \le \KH \NORMVH$, so that it would be  $\NORMLVH = 0$, iff
$ \SLVH = 0 $. But the $ \OPC (\VHI)  $'s are l.i., so that it should be $ \ASUBI = 0 $, $ 1 \le i \le \DIML $, which is a contradiction.

Since the $ \VHI $'s are linearly independent i.s.n., and they all belong to
$\KERPER$, use the Gram-Schmidt procedure, see e.g. \cite{BIRKHOFF}, p. 204,  to obtain $\DIML$ orthonormal elements 
$\UHI \in \KERPER \subseteq C( \THREEBETA ) \subseteq \HILBERT$, that span the same space than the $ \VHI $'s,
so that $ \| \UHI \|_{\HILBERT} = 1$, $ \PRODH { \UHI, \UHJ } = 0$ for $ i \neq j$,
and $ \PRODH { \UHI, w } = 0$, for $ 1 \le i \le \DIML $ and $ \forall \; w \in \KERC$.
Hence, each $\UHI$ is a linear combination i.s.n. of the $\VHI$'s, and since this transformation is invertible, then each $\VHI$ is a linear transformation i.s.n. of the $\UHI$'s. The vectors $\OPC (\UHI)$ are linearly independent: if not, there would exist
$\ASUBI$'s,  $\ASUBI \in \REALS $ for $ 1 \le i \le \DIML $, not all null, such that $ \SLUH \; = \; 0 $, but then, 
$ \OPC ( \SUH ) = 0$, so that $ \SUH \in \KERC$. Then, for $ 1 \le k \le \DIML$, we have  $ \PRODH { \UHK, \SUH } = 0 $.
But $ \PRODH { \UHK, \SUH } = \sum_{i=1}^{\DIML} \ASUBI  \PRODH { \UHK, \UHI } = 
\ASUBK  {\| \UHK \|^{2}_{\HILBERT}} =
\ASUBK$, so that $\ASUBK = 0$ for $ 1 \le k \le \DIML $, which is a contradiction.

Call $ \IL $ the span of the linearly independent vectors $ \OPC (\PIHI )$, for $ 1 \le i \le \DIML $, so that 
$ \IMAGE [ \THREEBETA) ] \subseteq \IL  $. Since $ \OPC (\PIHI )= \OPC (\VHI )$, then $\IL$ is the span of the linearly independent vectors $ \OPC (\VHI )$, for $ 1 \le i \le \DIML $. Since each $ \VHI $ is a linear combination i.s.n. of the 
linearly independent i.s.n. elements $ \UHI $, then, since $\OPC$ is a linear operator, each vector $\OPC (\VHI )$ is a linear combination of the linearly independent vectors $ \OPC (\UHI )$ and vice-versa, and then $ \IL$ is the span of the 
linearly independent vectors $ \OPC (\UHI )$, for $ 1 \le i \le \DIML $. 

Call $ \IMAGE [ S(\THREEBETA) ]$ the image of the operator $ \OPS: S(\THREEBETA) \to \RPAR $. Recall that
if $ u \in S(\THREEBETA) $ then $\OPS (u) = \OPC (u)$. Clearly, $ \IL \subseteq \IMAGE [ S(\THREEBETA) ]$.
If $ u \in S(\THREEBETA) $ then $u$ is a linear combination i.s.n.
of a finite number of elements in $\THREEBETA$, and then the vector $\OPC (u) \in \IMAGE [ S(\THREEBETA) ] $,  is the same
linear combination of the corresponding finite number of vectors in $\IMAGE [ \THREEBETA ]$. 
But since each vector
in $\IMAGE [ \THREEBETA ]$ is a linear combination of the vectors $\OPC (\UHI)$, for $ 1 \le i \le \DIML $,
then $ \OPC (u)$ is a linear combination of the independent vectors $\OPC (\UHI)$, for $ 1 \le i \le \DIML $, hence $  \OPC (u) \in \IL$, so that 
$\IMAGE [ S(\THREEBETA) ] = \IL$. 
If $ u \in  S(\THREEBETA) $, then since  $ \IMAGE [ S(\THREEBETA) ] = \IL $, 
there exist $\ALFAIU \in \REALS$, $ 1 \le i \le \DIML $, such that $ \OPC (u) = \sum_{i=1}^{\DIML} \ALFAIU \; \OPC( \UHI ) $.
Define $ \GAMAU = u - \sum_{i=1}^{\DIML} \ALFAIU \;  \UHI $, then $ \OPC (\GAMAU) = 0$, so that $\GAMAU \in \KERC$ and
$ \sum_{i=1}^{\DIML} \ALFAIU \;  \UHI \in \KERPER $. Hence $u \in S(\THREEBETA)$ may be written as
$ u  = \sum_{i=1}^{\DIML} \ALFAIU \;  \UHI  + \GAMAU \; \ISN$
\begin{OBSERVATION}
Note that, whenever $ u \in C ( \THREEBETA ) $ may be written as $ u  = \sum_{i=1}^{\DIML} \ALFAIU \;  \UHI  + \GAMAU $, where
$ \GAMAU \in \KERC$ and the $ \UHI$'s, $ 1 \le i \le \DIML$,  are orthonormal elements in $ \KERPER $, then we have
\begin{equation} \label{eq:normsumh}
\NORMH{ u }^{2} = \sum_{i=1}^{\DIML}  \left| \ALFAIU \right|^{2} +  \NORMH{ \GAMAU }^{2}
\end{equation}
\end{OBSERVATION}
Call $ \IMAGE [ C(\THREEBETA) ]$ the image of the operator $ \OPC: C(\THREEBETA) \to \RPAR $.
Since\break $ \IMAGE [ S(\THREEBETA) ] \subseteq  \IMAGE [ C(\THREEBETA) ]$ then $ \IL \subseteq \IMAGE [ C(\THREEBETA) ]$.
Fix an arbitrary $ u \in C(\THREEBETA) $, then there exists a sequence $ (u_{n})_{n \in \NAT}$ of elements $ u_{n} \in S(\THREEBETA)$ such that $ \| u_{n} - u \|_{\HILBERT} \to 0 $, 
and $ \| \OPC (u_{n}) - \OPC (u) \|_{\RPAR} \to 0 $. 
Since $ \| u_{n} - u \|_{\HILBERT} \to 0 $,
then $ {\left( u_{n} \right)}_{n \in \NAT} $ is a Cauchy fundamental sequence in $\HILBERT$.
Since $ u_{n} \in S(\THREEBETA)$, then, there exist sequences $ {\left( \ALFAIUN \right)}_{n \in \NAT} $, with $\ALFAIUN \in \REALS$  for $ 1 \le i \le \DIML $, $ \forall n \in \NAT $, and a sequence $ {\left( \GAMAIUN \right)}_{n \in \NAT} $, with $\GAMAIUN \in \KERC$, $\forall n \in \NAT$, 
such that $ u_{n}  = \sum_{i=1}^{\DIML} \ALFAIUN \;  \UHI  + \GAMAIUN$.
Since $ {\left( u_{n} \right)}_{n \in \NAT} $ is a Cauchy fundamental sequence in $\HILBERT$,
with $ u_{n} \in S(\THREEBETA) \subseteq  C(\THREEBETA)$,
then, (\ref{eq:normsumh})
shows that the sequences $ {\left( \ALFAIUN \right)}_{n \in \NAT} $ for $ 1 \le i \le \DIML $ are Cauchy fundamental sequences of real numbers, and the sequence $ {\left( \GAMAIUN \right)}_{n \in \NAT} $ is a Cauchy fundamental sequence
of elements in $ \KERC \subseteq C(\THREEBETA) \subseteq \HILBERT$. 
Since the reals are complete, there exist real numbers $\BETAZI \in \REALS$ for which $\ALFAIUN \to \BETAZI$ for $ 1 \le i \le \DIML $,
and, since $ \HILBERT $ is complete and $ \KERC $ is closed,
there exists an element $ \eta \in \KERC $ such that $ \| \GAMAIUN - \eta \|_{\HILBERT} \to 0 $.
Define $ u^{\prime} = \sum_{i=1}^{\DIML} \BETAZI \;  \UHI  + \eta$, then $ u^{\prime} \in  C(\THREEBETA)$. 
Then (\ref{eq:normsumh}) shows that $ \left\|  u_{n} - u^{\prime} \right\|_{\HILBERT} \to 0$.
Since $ \left\| u - u^{\prime} \right\|_{\HILBERT} \le  \left\| u - u_{n} \right\|_{\HILBERT} + \left\| u_{n} - u^{\prime} \right\|_{\HILBERT}$, taking the limit, we obtain $ u = u^{\prime} \; \ISN $
Hence for each $ u \in C(\THREEBETA)$ we have found real numbers $ \ALFAIU \in \REALS$, for $ 1 \le i \le \DIML $,
and an element $ \GAMAU \in \KERC$ such that
\begin{equation} \label{eq:udecomp}
u = \sum_{i=1}^{\DIML} \ALFAIU \;  \UHI  + \GAMAU \qquad \; \ISN \qquad \forall u \in C(\THREEBETA)
\end{equation}
Then, $ \OPC (u) = \sum_{i=1}^{\DIML} \ALFAIU \;  \OPC (\UHI) $, so that $ \OPC (u) \in \IL $, and then
$ \IMAGE [ C(\THREEBETA) ] = \IL $.
Additionally, since for arbitrary $ u \in C(\THREEBETA) $,
from (\ref{eq:udecomp}), we have\break
$ \PROJ ( u \mid \KERPER ) = \sum_{i=1}^{\DIML} \ALFAIU \;  \UHI   \; \ISN $, 
then, $ u \in \KERPER $ iff $ u = \sum_{i=1}^{\DIML} \ALFAIU \;  \UHI $\break $ \ISN $, and
then $\KERPER$ is a \emph{finite}
dimension subspace, $ \KERPER \subseteq C(\THREEBETA) \subseteq \HILBERT$, with dimension $\DIML$, even though
$\HILBERT$ might be a non-separable space. 

Hence, we have $ \IMAGE[ \KERC ] = \{ 0 \} $,
and $ \IMAGE [ \THREEBETA ] \subseteq  \IL = \IMAGE [ S(\THREEBETA) ]  = \IMAGE [ C(\THREEBETA) ] = \IMAGE[ \KERPER ] $. 
\subsubsection{Generalized Riesz representation of the operator $\OPC: C(\THREEBETA) \to \RPAR $ } 
From (\ref{eq:udecomp}), it is  $ \OPC (u) = \sum_{i=1}^{\DIML} \ALFAIU \;  \OPC (\UHI)$, for all $ u \in C(\THREEBETA) $.
Since the $ \UHI$'s are orthonormal and perpendicular to $\GAMAU$,
then, 
$ \PRODH { u , \; \UHK } \; =$\break
$ \PRODH { \sum_{i=1}^{\DIML} \ALFAIU \;  \UHI  + \GAMAU , \; \UHK } 
\; = \;
\ALFAKU 
$. Hence
\begin{equation} 
\nonumber  
\OPC (u ) \;= \; \sum_{i=1}^{\DIML} \PRODH { u, \UHI }  \OPC ( \UHI ) \qquad \forall u \in C(\THREEBETA)
\end{equation}
see (\ref{eq:REP}), which is the vector generalized Riesz representation for the extension $ \OPC $ of an
operator $ \OPZ: \THREEBETA \to \RPAR$, $ \THREEBETA \subseteq \HILBERT $, satisfying the OP-HRBB condition.
\section{Optimal estimator under the HRBB condition}    \label{sec:OPTEST}
The space $\SPACELTWOTRUE$ is a Hilbert space, \cite{HEWITT} p. 194,
with semi-inner product 
$ \PRODH { u_{1} , u_{2} } = \PRODLTWO { u_{1} , u_{2} } = \int u_{1} \; u_{2} \; \DIFF \RPTRUE $, $ \forall \, u_{1},  u_{2} \in \SPACELTWO $,
semi-norm $ \NORMH{ u } = \NORMLTWO{ u } = \bigl( \int u^{2} \; \DIFF \RPTRUE \bigr)^{1/2} $,
and equality in semi-norm (i.s.n.) given by equality with
probability 1 (w.p. 1). 
\begin{LEMMA}  \label{lemma:HRBB-EST}
If the HRBB condition holds for Problem \ref{prob:MAIN},
see Definition \ref{def:HRBB},
then there exists  
a finite covariance unbiased estimator $ \PSIHC \in \UG $.
\end{LEMMA}
\begin{proof}
Since $ \SPACELTWO $ is a Hilbert space, then we take the elements of $ \HILBERT $ as the functions in $ \SPACELTWO $.
Define the operator  $ \OPZ ( u ) = \FUNCH \bigl( \PI^{-1} ( u ) \bigr) $, for all $ u \in \THREEBETA $,
see Section \ref{sub:BETA-OP}. 
Since the HRBB condition 
holds for Problem \ref{prob:MAIN}, see (\ref{eq:HRBB}), then the OP-HRBB condition holds, 
see (\ref{eq:OP-HRBB}), and then we may apply Theorem \ref{theo:OP-HRBB}. 
From (\ref{eq:REP}) we obtain 
\begin{align}   \label{eq:OP-REP}
\OPC (u) 
&= \sum_{i=1}^{\DIML} \PRODH { u, \UHI } \OPC ( \UHI )
= \sum_{i=1}^{\DIML} \OPC ( \UHI) \int u \; \UHI  \; \DIFF \RPTRUE 
\\
\nonumber
&= \int u \; \left[ \sum_{i=1}^{\DIML} \UHI \; \OPC(\UHI) \right] \DIFF \RPTRUE \qquad \forall u \in C ( \THREEBETA ) 
\end{align}
where the $ \UHI $'s are orthonormal, with $ \UHI \in \KERPER \subseteq C( \THREEBETA )
$, for $ 1 \le i \le \DIML $.
Note the importance of working with finite dimensions $ \DIMPAR $ and $ \DIML $, with $ 1 \le \DIML \le \DIMPAR $, since this permits exchanging sums and integrals invoking elementary properties of Lebesgue integrals.
Define
\begin{equation} \label{eq:HRBB-EST}
\FIHC = \sum_{i=1}^{\DIML} \UHI \; \OPC(\UHI)
\end{equation}
so that 
$
\OPC (u) = \int \FIHC \; u \; \DIFF \RPTRUE 
$, $ \forall \; u \in C ( \THREEBETA ) $.

Since each $\OPC(\UHI) $ is some constant
real vector, i.e. $\OPC(\UHI) \in \RPAR $, for $ 1 \le i \le \DIML $, and each $\UHI \in \SPACELTWOTRUE $, then each component of the vector 
$ \FIHC $ is square integrable, i.e. $ \left[ \FIHC \right]_{i} \in \SPACELTWOTRUE $, equivalently $ \FIHC \in \VSPACETWOTRUE $. 
Then $ \FIHC $ is a measurable function from $ \VSPACETWOTRUE $ to $ \RPAR $, so that
$ \FIHC (\SAMP)= \sum_{i=1}^{\DIML} \UHI (\SAMP) \; \OPC(\UHI) $ is a random vector, 
$ \FIHC (\SAMP) : \SOM \to \RPAR $, that does not depend on the sub-indexes $\INDX \in \SINDX $.
Additionally since  $ \FIHC \in \VSPACETWOTRUE $, then  $ \PSIHC = \FIHC + \GTRUE $ has finite
covariance as previously discussed in Section \ref{subsec:MTS}.
Since $ \OPC (u) = \OPZ (u) $ if $ u \in \THREEBETA $, see Section \ref{subsub:EXTC}, and, for each $ u \in \THREEBETA $ there exists 
$ \INDX \in \SINDX $ such that $ u = \PIT $, see Hypothesis \ref{hypo:BASIC}, 
and $ \OPZ(\PIT) = \HT $, $ \forall \INDX \in \SINDX $, see (\ref{eq:OPZ}),
then, using (\ref{eq:OP-REP}) and (\ref{eq:HRBB-EST}),
$
\HT = \OPZ ( \PIT ) = \OPC (\PIT) = \int  \FIHC \; \PIT \; \DIFF \RPTRUE =  \int  \FIHC \; \left( d \RPT / d \RPTRUE \right) \; \DIFF \RPTRUE 
= \int \FIHC \; \DIFF \RPT = \int \FIHC (\SAMP) \; \DIFF  \SOMPT = \EXPT \left[ \FIHC \right]
$, $ \forall \INDX \in \SINDX $, see (\ref{eq:UNBIASED}).
Then, $ \EXPT \left[ \PSIHC \right] = \int \PSIHC \; \DIFF \RPT = \int \PSIHC (\SAMP) \; \DIFF  \SOMPT = \GT $,
$ \forall \INDX \in \SINDX $.
Hence, $ \PSIHC (\SAMP) $ is unbiased for all $ \INDX \in \SINDX $, and then $ \PSIHC \in \UG $.
\end{proof}
\begin{DEFINITION} \label{def:HRBB-EST}
Define the HRBB estimator as $ \PSIHC = \FIHC + \GTRUE $, 
where $ \FIHC $ is given by (\ref{eq:HRBB-EST}) as discussed in Lemma \ref{lemma:HRBB-EST}, so that $ \PSIHC  \in \UG $.
\end{DEFINITION}
\begin{DEFINITION}[Barankin-efficient estimator]     \label{def:BKEFF}
A finite covariance unbiased estimator $ \PSIH \in \UG $ for Problem \ref{prob:MAIN}, will be called Barankin-efficient, if
$ \COVTRUE ( \PSI ) \ge \COVTRUE ( \PSIH )$, for all $ \PSI \in \UG $.
Equivalently, $ \PSIH $ is a minimum-covariance unbiased estimator for Problem \ref{prob:MAIN}.
\end{DEFINITION}
\begin{DEFINITION}    \label{def:MSUP}
Let $ \mathscr{W} $ be a collection of real s.n.n.d. matrices of dimensions $ N \times N $. 
A s.n.n.d. matrix $ A \in \REALS^{N \times N} $ is an upper (lower) bound for $ \mathscr{W} $  if $ A \ge W $ ($ A \le W $),
$ \forall W \in \mathscr{W} $.
Define, if it exists, the
matrix-supreme (msup) of the matrices in $ \mathscr{W} $, as a real s.n.n.d. matrix $A$ of dimensions $ N \times N$,
such that $ A \ge W $, $ \forall W \in \mathscr{W} $, and such that for each $ \epsilon \in \REALS^{+}$,
$ \epsilon > 0$, there exists $ W(\epsilon) \in  \mathscr{W} $ such that 
$
\| A - W(\epsilon) \|_{F} < \epsilon
$.
The notation will be $ A = \underset{ W \in \mathscr{W} } {\MSUP} \; \mathscr{W} $. If $ A \in \mathscr{W} $, then A will be called 
the matrix-maximum of $ \mathscr{W} $.
\end{DEFINITION}

Define $ \MATRIXL $ as the real matrix, $ \MATRIXL \in \REALS^{ \DIMPAR \times \DIML } $, 
with columns $ \left[ \MATRIXL \right]_{i} = \OPC(\UHI) $, for $ 1 \le i \le \DIML $,
and define $ {\UVECH}^{T} = \left( \UH_{1}, \UH_{2}, \cdots, \UH_{\DIML} \right)$,
so that  $ \FIHC = \MATRIXL \; \UVECH$, see (\ref{eq:HRBB-EST}).
Since the $\UHI$'s are orthonormal in $ \SPACELTWOTRUE $, then $ \EXPTRUE \left[ \UVECH \; \UVECH^{T} \right] = \IDL $, where $ \IDL $ is the identity matrix of dimensions $ \DIML \times \DIML $. 
From (\ref{eq:HRBB-EST}), we have:
$
\COVTRUE ( \PSIHC ) = \EXP_{\INDXTRUE} \left[ \FIHC \; \FIHC^{T} \right] = 
\EXP_{\INDXTRUE} \left[ \MATRIXL \; \UVECH \; \left( \MATRIXL \; \UVECH \right)^{T} \right] =
\MATRIXL \; \EXP_{\INDXTRUE} \left[  \UVECH \;  \UVECH^{T} \right] \MATRIXL^{T} = 
\MATRIXL \; \IDL \; \MATRIXL^{T} = \MATRIXL \; \MATRIXL^{T}
$
so that
\begin{equation} \label{eq:COV}
\COVTRUE ( \PSIHC ) = 
\MATRIXL \; \MATRIXL^{T}
= \sum_{i=1}^{\DIML} \OPC(\UHI)\; \OPC(\UHI)^{T}
\end{equation}
\begin{THEOREM}      \label{theo:HRBB-OPT}
If the HRBB condition holds for Problem \ref{prob:MAIN},
see Definition \ref{def:HRBB},
then the HRBB estimator $ \PSIHC $, see Definition \ref{def:HRBB-EST},
is an unbiased Barankin-efficient estimator, and $ \COVTRUE \left( \PSIHC \right) = \underset{ W \in \WA} {\MSUP} \; \WA  $.
\end{THEOREM}
\begin{proof}
The HRBB estimator $ \PSIHC $ is unbiased and has finite covariance as a consequence of Lemma \ref{lemma:HRBB-EST}. 
To show that it is Barankin-efficient
let's consider the following two cases.
\\
\indent
1) {\em All the $ \UHI$'s belong to $ S ( \THREEBETA ) $}.
Then,  $ \UVECH \in \bigl( S ( \THREEBETA ) \bigr)^{\DIML} $. Since $ \FIHC = \MATRIXL \; \UVECH $, then
$ \FIHC  \in \bigl( S ( \THREEBETA) \bigr)^{\DIMPAR} $, and then, see Theorem \ref{theo:ONE} and Observation \ref{obs:PARORD},
we have equality in (\ref{eq:THEONE}). 
More precisely, since each $ \UHI \in S ( \THREEBETA ) $, then, there exist $ M_{i} \in \NAT $, $ \AVEC_{i} \in \REALS^{  M_{i} } $, and
$ \TAO_{i} \in \SINDX^{M_{i}} $, such that $ \UHI = \AVEC_{i}^{T} \VECB ( \TAO_{i} ) $ w.p. 1, for $ 1 \le i \le \DIML $.
Define $ \widehat{M} = \sum_{i=1}^{\DIML} M_{i} $, and $ {\widehat{\TAO}}^{T} = \left( \TAO_{1}^{T} \cdots  \TAO_{\DIML}^{T} \right) $,
$ \widehat{\TAO} \in {\SINDX}^{\widehat{M}} $. 
Call $ \widehat{\VECB} = \VECB ( \widehat{\TAO} ) $, $ \widehat{\VECB} \in {\THREEBETA}^{\widehat{M}} $,
so that 
$ \widehat{\VECB}^{T} = \VECB^{T} (\widehat{\TAO}) = \left( \VECB^{T}(\TAO_{1}) \cdots \VECB^{T}(\TAO_{\DIML}) \right) $,
and $ \widehat{B} = \EXPTRUE \left[ \widehat{\VECB} \; \widehat{\VECB}^{T} \right] $, $ \widehat{B} \in \REALS^{ \widehat{M} \times \widehat{M} } $.
Define the real matrix $ \widehat{A} \in \REALS^{ \DIML \times \widehat{M} } $, as the block-diagonal matrix
$  \widehat{A} = \DIAG \left( \AVEC_{1}^{T}, \AVEC_{2}^{T}, \dots, \AVEC_{\DIML}^{T} \right) $, 
where each block $ \AVEC_{i}^{T} $ is of dimension $ 1 \times \ M_{i}$, for $ 1 \le i \le M_{i} $,
so that $ \UVECH =  \widehat{A} \;  \widehat{\VECB} $ w.p. 1. 
Since the $\UHI$'s are orthonormal in $ \SPACELTWOTRUE $, we have $ \EXPTRUE \left[ \UVECH \; \UVECH^{T} \right] = \IDL $. 
Since $ \EXPTRUE \left[ \UVECH \; \UVECH^{T} \right] =  
\EXPTRUE \left[ \widehat{A} \;  \widehat{\VECB}  \; \left( \widehat{A} \;  \widehat{\VECB} \right)^{T} \right] =
\widehat{A} \; \EXPTRUE \left[ \widehat{\VECB}  \; \widehat{\VECB}^{T} \right] \widehat{A}^{T} = 
\widehat{A} \; \widehat{B} \; \widehat{A}^{T}
$, then $ \widehat{A} \; \widehat{B} \; \widehat{A}^{T} = \IDL $, so that $ \DET \left( \widehat{A} \; \widehat{B} \; \widehat{A}^{T} \right) = 1$.
Since $ \PSIHC \in \UG $ is unbiased, then, see Observation \ref{obs:MATG}, $ G ( \TAOH )  = \EXPTRUE \left[ \FIHC \; \VECB^{T}(\TAOH) \right] = 
\EXPTRUE \left[ \MATRIXL \; \UVECH \; \VECB^{T}(\TAOH) \right] =  \MATRIXL \; \EXPTRUE  \left[ \UVECH \; {\widehat{\VECB}}^{T} \right] = 
\MATRIXL \; \EXPTRUE  \left[  \widehat{A} \;  \widehat{\VECB}  \; {\widehat{\VECB}}^{T} \right] =
\MATRIXL \widehat{A} \; \EXPTRUE  \left[  \widehat{\VECB}  \; {\widehat{\VECB}}^{T} \right] = \MATRIXL \; \widehat{A} \; \widehat{B} 
$.
Take $ \QUADQH = \left(  \widehat{M}, \DIML, \widehat{A}, \TAOH \right) $, see Definition \ref{def:CONDA}, so that $ \QUADQH \in \CONDA $ 
since $ \DET \left( \widehat{A} \; \widehat{B} \; \widehat{A}^{T} \right) = 1$,
and then
$ W ( \QUADQH ) \in \WA $.
Hence $ W ( \QUADQH ) = G ( \TAOH ) \;  \widehat{A}^{T} \left( \widehat{A} \; \widehat{B} \; \widehat{A}^{T} \right)^{-1} \widehat{A} \; G^{T}( \TAOH ) =
G ( \TAOH ) \;  \widehat{A}^{T} \widehat{A} \; G^{T}( \TAOH ) =
\MATRIXL \; ( \widehat{A} \; \widehat{B} \widehat{A}^{T} ) ( \widehat{A} \; \widehat{B}^{T} \widehat{A}^{T} ) \MATRIXL^{T}
= \MATRIXL \;  \MATRIXL^{T} = \COVTRUE ( \PSIHC )
$, see (\ref{eq:COV}). Then $  \COVTRUE ( \PSIHC ) \in \WA  $, and then, see (\ref{eq:THEONE}) and Observation \ref{obs:PARORD},
$ W ( \QUADQH ) = \COVTRUE ( \PSIHC ) $ is a matrix-maximum for the matrices $ W \in \WA $
and a matrix-minimum for the covariances of any unbiased estimator $ \PSI \in \UG$, 
so that $ \PSIHC $ is a minimal covariance unbiased estimator, i.e.
the unbiased HRBB estimator $ \PSIHC $ is Barankin-efficient.
\\
\indent
2) {\em At least for one $i^{*}$, $ 1 \le i^{*} \le \DIML $, we have that $ \UH_{i^{*}}$ belongs to $ C ( \THREEBETA ) $  and $ \UH_{i^{*}} \notin S ( \THREEBETA ) $}.
Since each $ \UHI $ belongs to $ C ( \THREEBETA ) $, then there exist sequences $ \left( \YIM \right)_{m \in \NAT} $, for $ 1 \le i \le \DIML$,
with $ \YIM \in S ( \THREEBETA ), 1 \le i \le \DIML,  \forall m \in \NAT$, 
such that $ \underset{m \to \infty}{\lim} \NORMLTWO{ \UHI - \YIM } = 0 $, for $ 1 \le i \le \DIML$.
As before, for each $ \YIM \in S ( \THREEBETA ) $, there exist $ M_{i}(m) \in \NAT$, $ \AVEC_{i}(m) \in \REALS^{ M_{i}(m) } $,
and $ \TAO_{i}(m) \in \SINDX^{M_{i}(m)} $, such that $ \YIM = \AVEC^{T}_{i}(m) \; \VECB( \TAO_{i}(m) ) $ w.p. 1,
and $ \underset{m \to \infty}{\lim} \NORMLTWO{ \UHI - \AVEC^{T}_{i}(m) \; \VECB \left( \TAO_{i}(m) \right) } = 0 $.
Define $ \widehat{M} (m) = \sum_{i=1}^{\DIML} M_{i}(m)$, 
define $ \TAOH^{T} (m) =  \left( \TAO^{T}_{1}(m) \cdots \TAO^{T}_{\DIML}(m) \right)  $,\break
$ \TAOH (m) \in \SINDX^{ \widehat{M} (m) } $,
and $ \widehat{\VECB}^{T}_{m} = \VECB^{T}(\TAOH(m)) =  \bigl( \VECB^{T}(\TAO_{1}(m)) \cdots  \VECB^{T}(\TAO_{\DIML}(m)) \bigr) 
$, 
$ \widehat{\VECB}_{m} \in \THREEBETA^{ \widehat{M} (m) } $.
Define the real matrix $ \widehat{A}(m) \in \REALS^{ \DIML \times \widehat{M}(m) } $, as the block-diagonal matrix
$  \widehat{A}(m) = \DIAG \left( \AVEC_{1}^{T}(m), \ldots, \AVEC_{\DIML}^{T}(m) \right) $, 
where each block $ \AVEC_{i}^{T}(m) $ is of dimension $ 1 \times  M_{i}(m)$, for $ 1 \le i \le \DIML $.
Define $ \SVECH^{T}(m) = \left( {\widehat{s}}_{1}(m) \cdots {\widehat{s}}_{\DIML}(m) \right) $,
$ \SVECH(m)  \in { \Bigl( S (\THREEBETA ) \Bigr)}^{ \DIML } $,
so that
$ \SVECH(m) =  \widehat{A}(m) \;  \widehat{\VECB}_{m} $ w.p. 1.
Define $ \MATRIXS (m) = \EXPTRUE \left[  \SVECH(m) \; \SVECH^{T}(m) \right]$, $\MATRIXS (m) \in \REALS^{ \DIML \times \DIML }$.
Then, $ \MATRIXS (m) = \EXPTRUE \left[  \SVECH(m) \; \SVECH^{T}(m) \right] =$\break
$
\widehat{A}(m) \; \EXPTRUE \left[   \widehat{\VECB}_{m}  \widehat{\VECB}^{T}_{m}   \right] \widehat{A}^{T}(m)
$.
Call $ \widehat{B}(m) = \EXPTRUE \left[   \widehat{\VECB}_{m}  \widehat{\VECB}^{T}_{m}   \right] $, 
$ \widehat{B}(m) \in \REALS^{ \widehat{M} (m) \times \widehat{M} (m) } $,
so that
$ \MATRIXS (m) = \widehat{A}(m) \; \widehat{B}(m) \; \widehat{A}^{T}(m) $.
Since $ \PSIHC $ is unbiased, see Observation \ref{obs:MATG}, we have 
$ G \left( \TAOH (m) \right)
=
\EXPTRUE \left[ \FIHC \; \VECB^{T}(\TAOH(m)) \right]
=
\EXPTRUE \left[ \MATRIXL \; \UVECH \; \VECB^{T}(\TAOH(m)) \right]
=\break
\MATRIXL \; \EXPTRUE  \left[ \Bigl( \SVECH(m) + [  \UVECH - \SVECH(m) ] \Bigr) \; {\widehat{\VECB}}^{T}_{m} \right]
=
\MATRIXL \; \EXPTRUE  \left[ \left( \widehat{A}(m) \;  \widehat{\VECB}_{m}  + [  \UVECH - \SVECH(m) ] \right) \; {\widehat{\VECB}}^{T}_{m} \right]
=\break
\MATRIXL \; \widehat{A}(m) \widehat{B}(m) + \MATRIXL \; \EXPTRUE  \left[ \left( \UVECH - \SVECH(m)  \right) \; {\widehat{\VECB}}^{T}_{m} \right]
$.

Define  $ \QUADQH (m) = \left( \widehat{M}(m) , \DIML,  \widehat{A}(m), \TAOH (m) \right) $,
then, see Appendix Lemma \ref{lemma:MATRICES}, for $ m \ge M_{0} $, we have $ \DET \left( \widehat{A}(m) \; \widehat{B}(m) \; \widehat{A}^{T}(m)   \right) \neq 0 $,
so that, for $ m \ge M_{0} $, $ \QUADQH (m) \in \CONDA $, and then $ W ( \QUADQH (m) ) \in \WA $.
Then, after some algebra, for $ m \ge M_{0} $ we obtain:
\begin{align}
\nonumber
W ( &  \QUADQH (m) ) = G ( \TAOH(m) ) \;  \widehat{A}^{T}(m) 
\\
\nonumber
& \qquad \left( \widehat{A}(m) \; \widehat{B}(m) \; \widehat{A}^{T}(m) \right)^{-1} \widehat{A}(m) \; G^{T}( \TAOH(m) ) 
\\
\nonumber
&=
\; \MATRIXL \;  \MATRIXS (m)  \; \MATRIXL^{T} 
+
\MATRIXL \; \EXPTRUE  \left[  \left( \UVECH - \SVECH(m) \right)  \SVECH^{T}(m)  \right] \MATRIXL^{T} 
\\
\nonumber
& \quad
+
\MATRIXL \; \EXPTRUE  \left[  \SVECH(m) \left( \UVECH - \SVECH(m) \right)^{T}  \right] \MATRIXL^{T} 
\\
\nonumber
& \quad
+
\MATRIXL  \; \EXPTRUE  \left[  \left( \UVECH - \SVECH(m) \right)  \SVECH^{T}(m)  \right] \; \bigl( \MATRIXS (m) \bigr)^{-1}
\\
\nonumber
& \quad \qquad
\EXPTRUE  \left[  \SVECH(m) \left( \UVECH - \SVECH(m) \right)^{T}  \right] \MATRIXL^{T}
\end{align}
so that, see Appendix Lemma \ref{lemma:MATRICES}, 
$ W ( \QUADQH (m) ) \to \MATRIXL \; \MATRIXL^{T} = \COVTRUE ( \PSIHC )$, see (\ref{eq:COV}),
component by component and then in Frobenius norm.

Unlike the previous case, 
if for some $i^{*}$, $ 1 \le i^{*} \le \DIML $, we have that $ \UH_{i^{*}}$ belongs to $ C ( \THREEBETA ) $  
and $ \UH_{i^{*}} \notin S ( \THREEBETA ) $, then  $ \COVTRUE ( \PSIHC ) = \MATRIXL \; \MATRIXL^{T} \notin \WA $.
If not, $ \COVTRUE ( \PSIHC ) \in \WA $, and then we have equality in (\ref{eq:THEONE}), so that,
see Theorem  \ref{theo:ONE}, $ [ \FIHC ]_{i} \in S ( \THREEBETA ) $, for each $ 1 \le i \le \DIML $.
But $ \FIHC = \MATRIXL \UVECH $, so that $ \MATRIXL^{T} \FIHC = \MATRIXL^{T} \MATRIXL \UVECH  $, 
with $ \MATRIXL = \bigl( \OPC ( \widehat{u}_{1} ) \cdots  \OPC ( \widehat{u}_{\DIML} ) \bigr) $. 
Since the $ \DIML $ columns of $ \MATRIXL $ are linearly independent then $ \DET ( \MATRIXL^{T} \MATRIXL ) \neq 0 $, if not
there exists $ \ALFAVEC \in \RPAR $, $ \ALFAVEC \neq 0 $,  such that $  \MATRIXL^{T} \MATRIXL \ALFAVEC = 0 $, so that
$  \ALFAVEC^{T} \MATRIXL^{T} \MATRIXL \ALFAVEC =  \NORMRPAR { \MATRIXL \ALFAVEC }^{2} = 0 $, and then $ \MATRIXL \ALFAVEC = 0 $, contradiction.
Hence $ \UVECH = \bigl( \MATRIXL^{T} \MATRIXL  \bigr)^{-1} \MATRIXL^{T} \FIHC $, and then $ [ \UVECH ]_{i} \in S ( \THREEBETA ) $, for each $ 1\le i \le \DIML $, contradiction.

Since $ \forall \PSI \in \UG $, see Theorem \ref{theo:ONE}, it is $ \COVTRUE ( \PSI ) \ge W ( \QUADQH (m) ) $, 
then $ \COVTRUE ( \PSI ) - W ( \QUADQH (m) ) \ge 0 $, so that, taking the limit,  see Appendix Lemma \ref{lemma:MATLIM}, we have
$ \COVTRUE ( \PSI ) - \COVTRUE ( \PSIHC ) \ge 0 $, and then $ \COVTRUE ( \PSI ) \ge \COVTRUE ( \PSIHC )$, $ \forall \; \PSI \in \UG $.
Then, even though $ \COVTRUE ( \PSIHC ) \notin \WA $, we have $ \COVTRUE ( \PSI ) \ge \COVTRUE ( \PSIHC )$, $ \forall \; \PSI \in \UG $,
and, see Theorem \ref{theo:ONE}, $ \COVTRUE ( \PSIHC ) \ge W $, $\forall \; W \in \WA $.  
Furthermore, as previously shown, there exists a sequence $ W_{m} \in \WA $, 
$ W_{m} \equiv  W ( \QUADQH (m) )$, such that  $ \NORMFROB { \COVTRUE ( \PSIHC ) - W_{m} }  \to 0 $. Hence, though
$ \COVTRUE ( \PSIHC ) $ is not a matrix-maximum for $ \WA $, it is a matrix-supreme for $ \WA $, and $ \COVTRUE ( \PSIHC ) $ is a 
matrix-minimum for all the covariances of the estimators in $ \UG $, so that $ \PSIHC $ is Barankin-efficient.
\end{proof}
\begin{OBSERVATION}
A key point in Theorem \ref{theo:HRBB-OPT} is that if $ \WA $ is bounded above, then, 
the optimal covariance $ \COVTRUE ( \PSIHC ) $ may be obtained as the matrix-supreme, see Definition \ref{def:MSUP}, of the matrices
$ W \in \WA $, see (\ref{eq:THEONE}).
\begin{equation} 
\nonumber
\COVTRUE ( \PSIHC ) =  \underset{ \QUADQ \in \CONDA } {\MSUP}
\; G(\TAO) \; A^{T} \; \left(A \; B(\TAO) \; A^{T}) \right)^{-1} \; A \; G^{T}(\TAO) 
\end{equation}
and the matrix-supreme will be a matrix-maximum if and only if $ \PSIHC - \GTRUE = \FIHC \in \Bigl( S ( \THREEBETA ) \Bigr)^{ \DIMPAR } $.
\end{OBSERVATION}
\section{LMI equivalent formulation} \label{sec:OLMI}
\subsection{Equivalence of the LMI bound and the HRBB condition}
The statement that the Barankin covariance lower bounds $ \WA $ are bounded above, is a disguised form
of the HRBB condition, as a matter of fact the converse is also true, see Lemma \ref{lemma:BOUND-HRBB} and Theorem \ref{theo:EQUIV} below.
\begin{LEMMA}  \label{lemma:BOUND-HRBB}
If the Barankin covariance lower bounds $ \WA $ are bounded above, i.e. the collection $ \WA $ is bounded, see Definition \ref{def:COLLBOUNDA},
then
the HRBB condition holds, see  Definition \ref{def:HRBB}.
\end{LEMMA}
\begin{proof}
Call $ \MATBOUND $ the bound for $ \WA $, i.e.
\begin{equation} \label{eq:EQUIV-ONE}
\MATBOUND \ge 
W ( \QUADQ ) = G(\TAO) \; A^{T} \; \left(A \; B(\TAO) \; A^{T}) \right)^{-1} \; A \; G^{T}(\TAO) 
\end{equation}
for all $ \QUADQ \in \CONDA $.
Since this is true for matrices $A$ of all sizes $\DIMA \in \NAT$ for a given $ \DIMM \in \NAT$,  
$ A \in \REALS^{ \DIMA \times \DIMM } $, in particular is true when $ \DIMA = 1 $, i.e. when $A$ has one single row.
Call $ {\VECA}^{T} = \left( a_{1}, a_{2}, \cdots, a_{\DIMM} \right) $ the single row, so that  $ A = {\VECA}^{T} $, with $ {\VECA} \in \REALS^{\DIMM} $. 
Then, $ A \; B(\TAO) \; A^{T} = {\VECA}^{T} \; B(\TAO) \; {\VECA} = 
{\VECA}^{T} \; \EXPTRUE \left[ \VECB(\TAO) \; \VECB^{T}(\TAO) \right] \; {\VECA} =  
\EXPTRUE \left[  {\VECA}^{T} \; \VECB(\TAO) \; \VECB^{T}(\TAO) \; {\VECA} \right]  =  
\EXPTRUE \left[  \left( {\VECA}^{T} \; \VECB(\TAO) \right)^{2}  \right]
=
\EXPTRUE \left[ \left( \sum_{i = 1}^{\DIMM} \ASUBI \; \PITI \right)^{2}  \right]
=$\break
$
{\NORMPITISQ}
$. 
Observe that $ {\VECA}^{T} \; B(\TAO) \; {\VECA}$ is a non-negative scalar, i.e.\break
$ {\VECA}^{T} \; B(\TAO) \; {\VECA} \in \REALS^{+}$, 
and since we assumed that $ \QUADQ \in \CONDA $ then 
$ {\VECA}^{T} \; B(\TAO) \; {\VECA} \neq 0 $, as a matter of fact 
$ {\VECA}^{T} \; B(\TAO) \; {\VECA}  = \NORMPITISQ > 0 $.
On the other hand, $ G(\TAO) \; A^{T} = G(\TAO) \; {\VECA} = \sum_{i = 1}^{\DIMM} \ASUBI \; \HTI $.
Then, (\ref{eq:EQUIV-ONE}) takes the form
\begin{displaymath}
\MATBOUND \ge  
\frac{\left( \sum_{i = 1}^{\DIMM} \ASUBI \; \HTI \right)   \left( \sum_{i = 1}^{\DIMM} \ASUBI \; \HTI \right)^{T} } {\NORMPITISQ}
\end{displaymath}
Hence,
\begin{displaymath}
\TRACE \left[ \MATBOUND \right] \ge 
\frac{ \TRACE \left[ \left( \sum_{i = 1}^{\DIMM} \ASUBI \; \HTI \right)   \left( \sum_{i = 1}^{\DIMM} \ASUBI \; \HTI \right)^{T} \right] }
{\NORMPITISQ}
\end{displaymath}
Call $ \KH =  \left( \TRACE \left[ \MATBOUND \right] \right)^{1/2}$. 
Since,
$
\TRACE \left[  \left( \sum_{i = 1}^{\DIMM} \ASUBI \, \HTI \right)   \left( \sum_{i = 1}^{\DIMM} \ASUBI \, \HTI \right)^{T} \right] 
=\break
{\NORMGTISQ}
$
then,
$ \KH \ge { {\NORMGTI} } \; / \;  { {\NORMPITI} } $.\break
Hence
$ { {\NORMGTI} \le \KH \; {\NORMPITI} } $, $ \forall \DIMM \in \NAT $, $ \forall \ASUBI \in \REALS $,\break
 $ 1 \le i \le \DIMM $,
$ \forall \TSUBI \in \SINDX $, $ 1 \le i \le \DIMM $, such that $ {\NORMPITI}  \neq 0 $.
If $  {\NORMPITI} = 0 $, then $ \sum_{i = 1}^{\DIMM} \ASUBI \; \PITI  =0 $ w.p. 1.
Take an arbitrary $ u^{*} \in \THREEBETA $, then, see Observation \ref{obs:THREEBETA}, $ \NORMLTWO{  u^{*} } \neq 0 $.
Call $ \INDX^{*} = \PI^{-1} (u^{*}) $.
Then $ \sum_{i = 1}^{\DIMM} \ASUBI \; \PITI  +  (1/n) \; u^{*} = (1/n) \; u^{*} $ w.p. 1, for all $ n \in \NAT $, so that
$ \NORMLTWO{ \sum_{i = 1}^{\DIMM} \ASUBI \; \PITI  +  (1/n) \; u^{*} } = \NORMLTWO{ (1/n) \; u^{*} } \neq 0 $, for all $ n \in \NAT $.
Hence the previously obtained inequality applies, 
$ \NORMRPAR{ \sum_{i = 1}^{\DIMM} \ASUBI  \HTI  +  (1/n)  \FUNCH \bigl( \INDX^{*} \bigr) }
\le
\KH  \NORMLTWO{ \sum_{i = 1}^{\DIMM} \ASUBI \PITI  +  (1/n) u^{*} } 
=
\KH \; (1/n) \NORMLTWO{ u^{*} }
$, so that\break 
$ \NORMRPAR{ \sum_{i = 1}^{\DIMM} \ASUBI \; \HTI  +  (1/n) \; \FUNCH \bigl( \INDX^{*} \bigr) } \to 0 $,
as $ n \to +\infty $.
Since\break 
$
 \NORMRPAR{ \sum_{i = 1}^{\DIMM} \ASUBI \HTI   }
=
\NORMRPAR{ \sum_{i = 1}^{\DIMM} \ASUBI \HTI  +  (1/n) \FUNCH \bigl( \INDX^{*} \bigr) - (1/n) \; \FUNCH \bigl( \INDX^{*} \bigr) }
\le\break
\NORMRPAR{ \sum_{i = 1}^{\DIMM} \ASUBI \; \HTI  +  (1/n)  \; \FUNCH \bigl( \INDX^{*} \bigr) }
+
(1/n) \; \NORMRPAR{ \FUNCH \bigl( \INDX^{*} \bigr) }
$
then, taking the limit $ n \to +\infty $, it results $ \NORMRPAR{ \sum_{i = 1}^{\DIMM} \ASUBI \; \HTI   } = 0 $.

Hence
$ { {\NORMGTI} \le \KH \; {\NORMPITI} } $, $ \forall \DIMM \in \NAT $, $ \forall \ASUBI \in \REALS $, $ 1 \le i \le \DIMM $,
$ \forall \TSUBI \in \SINDX $, $ 1 \le i \le \DIMM $, such that $ {\NORMPITI}  \ge 0 $, 
as required by Definition \ref{def:HRBB}, so that 
the HRBB condition holds.
\end{proof}
\begin{THEOREM}   \label{theo:EQUIV}
The HRBB condition holds, see Definition \ref{def:HRBB},
if and only if the collection $ \WA $ is bounded, see Definition \ref{def:COLLBOUNDA}. 
\end{THEOREM}
\begin{proof}
If HRBB holds, see Theorem \ref{theo:HRBB-OPT}, then there exists $ \PSIHC \in \UG $ such that $ \PSIHC $ is Barankin-efficient, and  then
$ \COVTRUE \big( \PSIHC \bigr) \ge \ W $, $ \forall \; W \in \WA $, so that $ \WA $ is bounded.
The converse follows as a consequence of Lemma \ref{lemma:BOUND-HRBB}.
\end{proof}
\begin{LEMMA}  \label{lemma:DIRECT}
If there exists a finite covariance unbiased estimator $\PSI(\SAMP)$  for $\GT$, $\forall \INDX \in \SINDX$
for Problem \ref{prob:MAIN}, then 
$ \WA $ is bounded above, see Definition \ref{def:COLLBOUNDA}.
\end{LEMMA}
\begin{proof}
Since a finite covariance unbiased estimator $ \PSI $ exists, then (\ref{eq:THEONE}) shows that 
$ \WA $ is bounded.
\end{proof}
\subsection{Other equivalent LMI bounds}
One of the key ideas in Barankin's paper is the use of the free coefficients $ \ASUBI$'s, see \cite{BARANKIN} p. 480, that here take the form of the matrices $A$'s.
As discussed in \cite{BARANKIN}, and here below, the matrices $A$ are not required for the determination of the optimal matrix bound, but, they are most useful when one needs to compare the
Barankin bound with other bounds, such as Cramer-Rao, Bhattacharyya, etc. For the scalar case see \cite{BARANKIN} Corollaries 5--1 p. 487 and
6--1 p. 488.
For the vector Cramer-Rao bound, compare the results here with e.g.
\cite{STOICA} and references there.
\begin{DEFINITION} \label{def:COLLBOUNDB}
Define the pair $  \PAIRD = \left( \DIMM, \TAO \right) $, where $ \DIMM \in \NAT $, and $ \TAO \in {\SINDX}^{\DIMM} $.
Define $ \CONDB $ as the collection of all the pairs
$ \PAIRD $ with $ \DET \bigl( B(\TAO) \bigr) \neq 0  $, with $ B(\TAO) $ as in Definition \ref{def:CONDA}, so that
\begin{displaymath}
 \CONDB = \left\{ \PAIRD : \forall  \; \DIMM \in \NAT,  \forall \; \TAO \in {\SINDX}^{\DIMM} , \; 
{\rm with} \; \DET \bigl( B(\TAO) \bigr)  \neq 0   \right\}
\end{displaymath}
Define $ \WB $ as the collection of matrices 
\begin{displaymath}
V ( \PAIRD ) = G(\TAO) \bigl( B(\TAO) \bigr)^{-1} \; G^{T}(\TAO) \qquad \forall \PAIRD \in \CONDB 
\end{displaymath} 
with $ G(\TAO) $ as in Definition \ref{def:CONDA}.
Equivalently $ \WB = \bigl\{ V ( \PAIRD ) : \PAIRD \in \CONDB \bigr\} $.

Define the function $ \GT $ as $ \THREEBETA $-compatible if whenever $ \sum_{i = 1}^{\DIMM} \ASUBI \; \PITI= 0 $ w.p. 1,
we have $ \sum_{i = 1}^{\DIMM} \ASUBI \; \HTI = 0 $, with $ \DIMM \in \NAT $, $ \ASUBI \in \REALS $, $ \TSUBI \in \SINDX $,
for $ 1 \le i \le \DIMM $.
\end{DEFINITION}
Note that if $ \GT $ is not $ \THREEBETA $-compatible then no unbiased estimator exists for $ \GT $, $ \forall \INDX \in \SINDX $,
for Problems \ref{prob:INTEQ}, \ref{prob:LTWO} or \ref{prob:MAIN}.
If $ \GT $ is $ \THREEBETA $-compatible, then for $ \TAO \in \SINDX^{\DIMM} $ and $ \VECA \in \REALS^{\DIMM} $, if
$ \VECA^{T} \; \VECB ( \TAO )  = 0 $, then $ G(\TAO) \; \VECA = 0 $,
and for $ A \in \REALS^{\DIMA \times \DIMM} $, if $ A \; \VECB ( \TAO ) = 0 $, then $  G(\TAO) A^{T} = 0 $.
Hence, if $ \TAO_{1} \in \SINDX^{\DIMA} $ and we have $ A \; \VECB ( \TAO ) = \VECB ( \TAO_{1} ) $,
then $  G(\TAO) A^{T} = G(\TAO_{1}) $.
\begin{THEOREM}     \label{theo:W-EQUIV}
The collection $ \WA $, see Definition \ref{def:COLLBOUNDA}, is bounded above, if and only if
the collection $ \WB $ is bounded above and $ \GT $ is $ \THREEBETA $-compatible.
\end{THEOREM}
\begin{proof}
Assume $ \WA $ is bounded. Then, there exists a s.n.n.d. matrix $ B_{1} \in \REALS^{ \DIMPAR \times \DIMPAR } $,
such that $ B_{1} \ge W ( \QUADQ ) $, $ \forall \QUADQ \in \CONDA $. Take an arbitrary $ \PAIRD^{\prime} \in \CONDB $, 
with $ \PAIRD^{\prime} = \bigl( \DIMM^{\prime}, \TAO^{\prime} \bigr) $, so that $ \DET \bigr( B(\TAO^{\prime}) \bigl) \neq 0 $. 
Define $ \QUADQ^{\prime} = \bigl( \DIMM^{\prime}, \DIMM^{\prime}, \IDMP, \TAO^{\prime} \bigr) $, where $ \IDMP $ is the identity matrix
of dimensions $ \DIMM^{\prime} \times \DIMM^{\prime} $. 
Since $ \DET \bigr(  \IDMP B(\TAO^{\prime})  \IDMP^{T} \bigl) = \DET \bigr( B(\TAO^{\prime}) \bigl) \neq 0 $, then $ \QUADQ^{\prime} \in \CONDA $,
and we have $ W ( \QUADQ^{\prime} ) = V ( \PAIRD^{\prime} ) $, so that $ B_{1} \ge V ( \PAIRD^{\prime} ) $, $ \forall \PAIRD^{\prime} \in \CONDB $,
and then $ \WB $ is bounded.
Since $ \WA $ is bounded, then the HRBB condition holds, see Lemma \ref{lemma:BOUND-HRBB}, and then  (\ref{eq:HRBB}) shows that
$ \GT $ is $ \THREEBETA $-compatible.
Conversely, assume $ \WB $ is bounded. Then, there exists a s.n.n.d. matrix $ B_{2} \in \REALS^{ \DIMPAR \times \DIMPAR } $,
such that $ B_{2} \ge V ( \PAIRD ) $, $ \forall \PAIRD \in \CONDB $.
Take an arbitrary $ \QUADQ^{\prime} \in \CONDA $, 
with $ \QUADQ^{\prime} = \bigl( \DIMM^{\prime}, \DIMA^{\prime}, A^{\prime}, \TAO^{\prime} \bigr) $,
with $ A^{\prime} \in \REALS^{ \DIMA^{\prime} \times \DIMM^{\prime} } $,
so that $ \DET \bigr( A^{\prime} \; B(\TAO^{\prime}) \; \left( A^{\prime} \right)^{T} \bigl) \neq 0 $.
As in the proof of the last part of Theorem \ref{theo:ONE}, obtain $ \DIMT^{\star} \in \NAT $, 
$ 1 \le \DIMT^{\star} \le \DIMM^{\prime} $, $ A^{\star} \in \REALS^{ \DIMM^{\prime} \times  \DIMT^{\star} } $, and
$ \TAO^{\star} \in \SINDX^{\DIMT^{\star}} $, by elimination of the components of the vector
$ \VECB ( \TAO^{\prime} ) $ 
which are linear combinations w.p. 1 of previous components,
so that $  \VECB ( \TAO^{\prime} ) =  A^{\star} \: \VECB ( \TAO^{\star} ) $,
with 
$ \DET \bigl( B ( \TAO^{\star} ) \bigr) = \DET \Bigl( \EXPTRUE \bigl[ \: \VECB ( \TAO^{\star} ) \; \VECB^{T} ( \TAO^{\star} ) \bigr] \Bigr)
\neq 0 $.

Then,  
$ B ( \TAO^{\prime} )  = \EXPTRUE \bigl[ \: \VECB ( \TAO^{\prime} ) \; \VECB^{T} ( \TAO^{\prime} ) \bigr]  
=
A^{\star} \; \EXPTRUE \bigl[ \: \VECB ( \TAO^{\star} ) \; \VECB^{T} ( \TAO^{\star} ) \bigr] \left( A^{\star} \right)^{T} 
=\break
A^{\star} \; B ( \TAO^{\star} ) \;  \left( A^{\star} \right)^{T} 
$.
Since $ \GT $ is  $ \THREEBETA $-compatible, then
$ G ( \TAO^{\prime} ) = G ( \TAO^{\star} ) \; \left( A^{\star} \right)^{T}   $.
Hence, 
\begin{align}
\nonumber 
W ( \QUADQ^{\prime} ) =& \;  G ( \TAO^{\prime} ) \left( A^{\prime} \right)^{T} 
\bigl( A^{\prime} \; B ( \TAO^{\prime} ) \; \left( A^{\prime} \right)^{T}  ) \bigr)^{-1}
A^{\prime}   \; G ( \TAO^{\prime} )
\\
\nonumber
=& \;
G ( \TAO^{\star} )  \left( A^{\star} \right)^{T}  \left( A^{\prime} \right)^{T} 
\bigl(  A^{\prime} A^{\star} B ( \TAO^{\star} )  \left( A^{\star} \right)^{T}  \left( A^{\prime} \right)^{T}  \bigr)^{-1}
A^{\prime} \; \; A^{\star} \; G^{T} ( \TAO^{\star} )
\end{align}
Define $ \PAIRD^{\star} = \bigl( \DIMT^{\star}, \TAO^{\star}  \bigr) $
so that $ V ( \PAIRD^{\star} ) = G ( \TAO^{\star} ) B^{-1} ( \TAO^{\star} ) G^{T} ( \TAO^{\star} ) $.
Then, the Appendix Lemma \ref{lemma:RAYLEIGH} shows that $ V ( \PAIRD^{\star} ) \ge W ( \QUADQ^{\prime} ) $, so that
$ B_{2} \ge V ( \PAIRD^{\star} ) \ge  W ( \QUADQ^{\prime} ) $. Hence
$ B_{2} \ge W ( \QUADQ^{\prime} ) $, for all $ \QUADQ^{\prime} \in \CONDA $, so that $ \WA $ is bounded.
\end{proof}
For an arbitrary symmetric
matrix $ W $ call $ \lambda_{M} ( W ) \in \REALS $ its greatest eigenvalue. 
The operator norm $ \NORMOP{ A } $ of a matrix $ A \in \REALS^{N \times M} $ is its greatest singular value, \cite{BHATIA} p. 12, i.e. the non-negative square root of the greatest eigenvalue of the matrix $ A^{T} A $, so that $ \NORMOP{ A } = \bigl( \lambda_{M} ( A^{T} A ) \bigr)^{1/2} $.
For s.n.n.d. matrices singular values and eigenvalues coincide, \cite{ZHAN} p. 19, so that if $ W \in \COLLW $, then
$ \NORMOP{ W } = \lambda_{M} ( W ) $. Define a k-identity matrix as a matrix of the form $ K \; I_{\DIMMAT} $ where $ K \in \REALS$ and $ I_{\DIMMAT} $ is the identity matrix 
of dimensions $ \DIMMAT \times \DIMMAT $.
Then, we have
\begin{LEMMA}    \label{lemma:K-ID}
If $ X $ is a s.n.n.d. matrix, $ X \in \REALS^{\DIMMAT \times \DIMMAT} $,
then,
for $ K \in \REALS $, we have
$ K \; I_{\DIMMAT} \ge X $ if and only if
$ K \ge  \lambda_{M} ( X ) $.
\end{LEMMA}
\begin{proof}
Since $ X $
is symmetric, then it is diagonalizable, so that there exist an orthogonal  matrix $ Q \in \REALS^{ \DIMMAT \times \DIMMAT } $, and a diagonal matrix
$ \LMATL \in \REALS^{ \DIMMAT \times \DIMMAT } $, such that 
$ X  = Q \; \LMATL \; Q^{T} $. 
Then, since $ \lambda_{M} ( X ) \; I_{\DIMMAT} \; \ge \LMATL $, we have 
$ \lambda_{M} ( X ) \; I_{\DIMMAT} = Q \; \lambda_{M} ( X ) \; I_{\DIMMAT} \; Q^{T} \ge Q \; \LMATL \; Q^{T} = X 
$. Hence, if $ K \ge  \lambda_{M} ( X ) $, then $ K \; I_{\DIMMAT} \ge \lambda_{M} ( X ) \; I_{\DIMMAT}  \ge X $.
Conversely if $ K \; I_{\DIMMAT} \ge X $, then $ K \; I_{\DIMMAT} \ge  Q \; \LMATL \; Q^{T}$, so that
$ Q^{T} K \; I_{\DIMMAT} Q \ge  \LMATL $, and since $ Q^{T} K \; I_{\DIMMAT} \; Q = K \; I_{\DIMMAT} $, then $ K \; I_{\DIMMAT} \ge  \LMATL $.
Hence $ K \ge \lambda_{M} ( X ) $.
\end{proof}
\begin{LEMMA}    \label{lemma:K-BOUNDED}
A non-empty collection $ \COLLW $ of s.n.n.d. matrices $ W \in \REALS^{\DIMMAT \times \DIMMAT} $ is upper bounded if and only if there exists $ K_{\COLLW} \in \REALS $, such that $ K_{\COLLW} \; I_{\DIMMAT} \ge W $, $\forall W \in \COLLW $.
\end{LEMMA}
\begin{proof}
If there exists $ K_{\COLLW} \in \REALS $, such that $ K_{\COLLW} \; I_{\DIMMAT} \ge W $, $\forall W \in \COLLW $,
then by definition $ K_{\COLLW} \; I_{\DIMMAT} $ is a matrix bound for $ \COLLW $, and then $ \COLLW $ is bounded.
Conversely, assume $ \COLLW $ is bounded.
Then there exists 
$ B_{\COLLW} \in \REALS^{ \DIMMAT \times \DIMMAT } $ s.n.n.d., such that $ B_{\COLLW} \ge W $, $\forall W \in \COLLW $. 
Since $ B_{\COLLW} $
is symmetric, then there exists the real maximum eigenvalue $ \lambda_{M} ( B_{\COLLW} ) \in \REALS $.
Take $ K_{\COLLW} \in \REALS$ such that $ K_{\COLLW} \ge \lambda_{M} ( B_{\COLLW} ) $. Then, from Lemma \ref{lemma:K-ID},
$ K_{\COLLW} \; I_{\DIMMAT} \ge B_{\COLLW} \ge W $, $\forall W \in \COLLW $.
\end{proof}
\subsection{Main Theorem}            \label{subsec:MAIN}
Collecting all the previous results, we have 
\begin{THEOREM}[Main Theorem]   \label{theo:MAIN}
The following statements for Problem \ref{prob:MAIN} are equivalent, meaning that if any one of them is true, then they are all true:
\begin{enumerate}
\item    \label{it:MAIN-1}
A finite covariance vector unbiased estimator exists, i.e. $ \UG $ is not empty.
\item    \label{it:MAIN-2}
A Barankin-efficient vector unbiased estimator exists.
\item    \label{it:MAIN-3}
The HRBB condition holds.
\item    \label{it:MAIN-4}
The collection $ \WA $ is bounded.
\item    \label{it:MAIN-5}
There exists $ K \in \REALS^{+} $ such that $ K \; \IDP \ge W $, $ \forall W \in \WA $.
\item    \label{it:MAIN-6}
The collection $ \WB $ is bounded, and $ \GT $ is $ \THREEBETA $-compatible.
\item    \label{it:MAIN-7}
There exists $ K \in \REALS^{+} $ such that $ K \; \IDP \ge W $, $ \forall W \in \WB $,  and $ \GT $ is $ \THREEBETA $-compatible.
\end{enumerate}
\end{THEOREM}
\begin{proof}
\ref{it:MAIN-3}) $ \Rightarrow $ \ref{it:MAIN-2}) $ \Rightarrow $ \ref{it:MAIN-1}) follows from Theorem \ref{theo:HRBB-OPT}, 
\ref{it:MAIN-3}) $ \Leftrightarrow $ \ref{it:MAIN-4}) from Theorem \ref{theo:EQUIV},
\ref{it:MAIN-1}) $ \Rightarrow $ \ref{it:MAIN-4}) from Lemma \ref{lemma:DIRECT}, 
\ref{it:MAIN-4}) $ \Leftrightarrow $ \ref{it:MAIN-5})
and \ref{it:MAIN-6}) $ \Leftrightarrow $ \ref{it:MAIN-7}) from Lemma \ref{lemma:K-BOUNDED}, 
finally \ref{it:MAIN-4}) $ \Leftrightarrow $ 6) follows from Theorem \ref{theo:W-EQUIV}.
\end{proof}
\begin{COROLLARY}
As a corollary, the collection $ \UG $ is empty iff $ \WA $ is not bounded, i.e. for each $ k \in \NAT $ there exists $ W^{\star}_{k} \in \WA $
such that $ \NORMFROB { W^{\star}_{k} } \ge \NORMOP { W^{\star}_{k} } = \lambda_{M} ( W^{\star}_{k} )  > k $, so that $ \lim_{k \to +\infty} \NORMFROB { W^{\star}_{k} } = +\infty $.
Note that $ \UG $ is empty either because there are no unbiased estimators for $ \GT $, $ \forall \INDX \in \SINDX $,
or if there exist, they don't have finite finite covariance matrix at $ \INDXTRUE $, see Definition \ref{def:COLLBOUNDA}.
\end{COROLLARY}

\setcounter{section}{0} 
\renewcommand\thesection{\Alph{section}}
\section{APPENDIX}        \label{APPENDIX} 

\begin{LEMMA} \label{lemma:MATRICES}
Let $ \HILBERT $ be an arbitrary Hilbert space. Let the elements $ \UI \in \HILBERT  $, for $ 1 \le i \le \DIML < +\infty $, be orthonormal so that 
$ \NORMLTWO { \UI } = 1 $, for $ 1 \le i \le  \DIML $,
and $ \PRODH{ \UI, \UJ } = 0 $, for $ i \neq j $, $ 1 \le i,j \le  \DIML $, and then $ \PRODH{ \UI, \UJ } = \KRON_{i,j} $.
Assume that for each $ \UI $, for $ 1 \le i \le \DIML $, there exist sequences $ \bigl( s_{i}(m) \bigr)_{m \in \NAT} $, 
with $ {s}_{i}(m) \in \HILBERT $, $ \forall m \; \in \NAT $, for $ 1 \le i \le \DIML $,
such that $ \underset{m \to \infty}{\lim} \NORMH{ \UI -  {{s}}_{i}(m)  } = 0 $,
for $ 1 \le i \le  \DIML $.

Define 
$ S (m)  \in \REALS^{ \DIML \times \DIML }$, as a matrix with i-th, j-th element
$ \bigl[ S (m) \bigr]_{i,j} = \PRODH  { s_{i} ( m ) , \; s_{j} ( m ) } $, for $ 1 \le i, j \le \DIML$, $ \forall m \in \NAT $.
Then:
\begin{enumerate}
\item  \label{it:ONE}
$\NORMFROB { S (m) - \IDL } \to 0 $, as $ m \to \infty $, 
where $ \IDL $ is the identity matrix of dimensions $\DIML \times \DIML $.
\item   \label{it:TWO}
$ \DET  ( S (m) ) \to 1 $, and $ \NORMFROB { S^{-1} (m) - \IDL } \to 0 $, as $ m \to \infty $.
\item    \label{it:THREE}
$ \underset{m \to \infty}{\lim} \PRODH { \UI - {s}_{i}(m), \; {s}_{j}(m)   } = 0 $, for all $ 1 \le i,j \le \DIML $.
\end{enumerate}
\end{LEMMA}
\begin{proof}
a) From the Cauchy-Schwarz inequality
\ifundefined{DRAFT}
\else
we obtain
\fi
$ \bigl| \PRODH { \UI - {s}_{i}(m), \UJ } \bigr| 
\le
\NORMH{ \UI - {s}_{i}(m) }
$, so that $ \underset{m \to \infty}{\lim} \PRODH { \UI - {s}_{i}(m), \UJ } = 0 $, for all $ 1 \le i,j \le \DIML $.

b) Also $ \bigl| \PRODH { \UI - {s}_{i}(m), \UJ - {s}_{j}(m) } \bigr| 
\le
{ \NORMH{ \UI - {s}_{i}(m) } \; \NORMH{ \UJ - {s}_{j}(m) } } $,
and then
$ \underset{m \to \infty}{\lim} \PRODH { \UI - {s}_{i}(m), \UJ - {s}_{j}(m)  } = 0 $, for all $ 1 \le i,j \le \DIML $.

c) We have $ \PRODH { {s}_{i}(m), {s}_{j}(m)  } = 
\PRODH { {s}_{i}(m)- \UI + \UI , \; {s}_{j}(m) - \UJ + \UJ }
=
$\break
$
\PRODH { {s}_{i}(m)- \UI , \; {s}_{j}(m) - \UJ } +
\PRODH { \UI , \; {s}_{j}(m) - \UJ } +
\PRODH { {s}_{i}(m)- \UI, \; \UJ } +$\break
$
\PRODH { \UI , \;  \UJ } 
$.
Taking the limit, and using a) and b),
$ \underset{m \to \infty}{\lim} \PRODH { {s}_{i}(m), {s}_{j}(m)  } = \KRON_{i,j} $, for all $ 1 \le i,j \le \DIML $.
Then $ S(m) \to \IDL $ component by component, and then in Frobenius norm. 
Since the determinant of a matrix  
is an algebraic sum of a finite number of products of a finite number of elements of the matrix, see \cite{BIRKHOFF} p. 319, then $ \DET \left( S (m) \right) \to \DET ( \IDL ) = 1$, so that $ \exists M_{0} \in \NAT $,
such that $ \forall m \ge M_{0} $, it will be $ \DET \left( S (m) \right)  \ge 1/2 $, and then $ \DET \left( S (m) \right)  \neq 0 $.
Similarly, since the elements of the inverse of a matrix are the quotients of algebraic sums of a finite number of products of
a finite number of elements of the matrix divided
the determinant, see \cite{BIRKHOFF} p. 325, then $ S^{-1}(m) $ has a limit $ A_{0} $ component by component, 
so that $ S (m) \; S^{-1}(m) \to \IDL \; A_{0} $, but since $ S (m) \; S^{-1}(m) = \IDL $, $ \forall m \in  \NAT $, then $ A_{0} = \IDL $,  and then $ S^{-1}(m) \to \IDL $ component by component as $ m \to \infty $, and then in Frobenius norm, or any other matrix norm,
so that we have shown 
items \ref{it:ONE}) and \ref{it:TWO}).

d) We have $ \PRODH { \UI - {s}_{i}(m), \; {s}_{j}(m)  } =
\PRODH { \UI - {s}_{i}(m), \; {s}_{j}(m) - \UJ + \UJ } 
=
\PRODH { \UI - {s}_{i}(m), \; {s}_{j}(m) - \UJ } +
\PRODH { \UI - {s}_{i}(m), \; \UJ  }
$.
Applying a) and b) we obtain
$ \underset{m \to \infty}{\lim} \PRODH { \UI - {s}_{i}(m), \; {s}_{j}(m)   } = 0 $, for all $ 1 \le i,j \le \DIML $,
so that we have shown item \ref{it:THREE}).
\end{proof}
\begin{LEMMA}   \label{lemma:MATLIM}
Let $ \bigl(  A_{n} \bigr)_{n \in \NAT} $ be a sequence of s.n.n.d. matrices 
$ A_{n} \in \REALS^{N \times N} $, $ A_{n} \ge 0 $, $ \forall n \in N $, 
such that there exists $ A \in \REALS^{ N \times N } $ for which 
$ A_{n} \to A $ c.b.c and then in Frobenius norm. Then the matrix $ A $ is s.n.n.d.
\end{LEMMA}
\begin{proof}
Take $ \ALFAVEC \in \REALS^{N} $ arbitrary, since $ N $ is finite, then  $ \lim_{n \to +\infty } \ALFAVEC^{T} A_{n} \ALFAVEC = \ALFAVEC^{T} A \ALFAVEC $.
But $ \ALFAVEC^{T} A_{n} \ALFAVEC \ge 0 $, $ \forall n \in \NAT $, so that $ \lim_{n \to +\infty } \ALFAVEC^{T} A_{n} \ALFAVEC \ge 0 $, 
and then $  \ALFAVEC^{T} A \ALFAVEC \ge 0 $, $ \forall \ALFAVEC \in \REALS^{N}$.
\end{proof}
The following lemma is a LMI weighted form of the Cauchy-Schwarz inequality for matrices,
\cite{CHIPMAN} p. 1093. For convenience, a proof is given here.
\begin{LEMMA}        \label{lemma:CAUCHY-SCHWARZ}
Let $ \DIMLM , \DIMLN, \DIMLP \in \NAT $.
Let
$ \LMATH \in \REALS^{ \DIMLM \times \DIMLM } $ be an arbitrary real s.p.d. matrix, and let $ \LMATX \in \REALS^{ \DIMLN \times \DIMLM } $
and $ \LMATY \in \REALS^{ \DIMLP \times \DIMLM } $,  be otherwise arbitrary real matrices such that
$ \DET ( \LMATY \LMATH \LMATY^{T} ) \neq 0 $.
Then 
\begin{displaymath}
\LMATX \; \LMATH \; \LMATX^{T} \; \ge 
\LMATX \; \LMATH \; \LMATY^{T} \; \bigl( \LMATY \; \LMATH \; \LMATY^{T} \; \bigr)^{-1} \LMATY \; \LMATH \; \LMATX^{T}
\end{displaymath}
with equality if and only if there exists $ \LMATL \in \REALS^{ \DIMLN \times \DIMLP } $, such that $ \LMATX = \LMATL \; \LMATY $,
if and only if
\begin{displaymath}
\LMATX = \LMATX \;  \LMATH \;  \LMATY^{T} \; \bigl( \LMATY \; \LMATH \; \LMATY^{T} \; \bigr)^{-1} \LMATY \; 
\end{displaymath}
\end{LEMMA}
\begin{proof}
Let $ \LMATL \in \REALS^{ \DIMLN \times \DIMLP } $.
Define $ \LMATT ( \LMATL ) = \left( \LMATX - \LMATL \; \LMATY \right) \; \LMATH \; \left( \LMATX - \LMATL \; \LMATY \right)^{T} $, so that
$ \LMATT ( \LMATL ) $ is s.n.n.d., $ \forall  \LMATL \in \REALS^{ \DIMLN \times \DIMLP } $.
Define 
\begin{equation}
\nonumber
\LMATR ( \LMATL ) 
= 
\Bigl( \LMATL - \LMATX \; \LMATH \; \LMATY^{T} \left( \LMATY \;\LMATH \;\LMATY^{T} \right)^{-1} \Bigr)
 \; \LMATY \;\LMATH \;\LMATY^{T}
  \Bigl( \LMATL - \LMATX \; \LMATH \; \LMATY^{T} \left( \LMATY \;\LMATH \;\LMATY^{T}\right)^{-1} \Bigr)^{T} 
\end{equation}
and $ \LMATD = \LMATX \; \LMATH \; \LMATX^{T} -
\LMATX \; \LMATH \; \LMATY^{T} \left( \LMATY \;\LMATH \;\LMATY^{T}\right)^{-1} \LMATY \; \LMATH \; \LMATX^{T}$.
Note that $ \LMATR ( \LMATL ) \ge 0 $, $ \forall  \LMATL \in \REALS^{ \DIMLN \times \DIMLP } $.
Then $ \LMATT ( \LMATL ) = \LMATR ( \LMATL ) + \LMATD \ge 0 $, $ \forall  \LMATL \in \REALS^{ \DIMLN \times \DIMLP } $.
For $ \LMATL_{0} = \LMATX \; \LMATH \; \LMATY^{T} \left( \LMATY \;\LMATH \;\LMATY^{T}\right)^{-1} $, we have $ \LMATR ( \LMATL_{0} ) = 0 $,
so that $ \LMATT ( \LMATL_{0} ) = D \ge 0 $, and then the LMI is obtained.
If there is equality then $ D = 0 $, and then $ \LMATT ( \LMATL ) =  \LMATR ( \LMATL ) $, $ \forall  \LMATL \in \REALS^{ \DIMLN \times \DIMLP } $.
In particular for $ \LMATL_{0} $ we have $ \LMATR ( \LMATL_{0} ) = 0 $, and then $ \LMATT ( \LMATL_{0} ) = 0 $. But, since $ \LMATH $ is s.p.d.
then $ \LMATX = \LMATL_{0} \; \LMATY = \LMATX \; \LMATH \; \LMATY^{T} \left( \LMATY \;\LMATH \;\LMATY^{T}\right)^{-1} \LMATY $. 
As for the converse, if there exists $ \LMATL_{1} $ such that $ \LMATX = \LMATL_{1} \; \LMATY$, then
$  \LMATT ( \LMATL_{1} ) = 0 $, since  $  \LMATR ( \LMATL_{1} ) \ge 0 $ by definition, and $ D \ge 0 $ as previously shown, then,  $  \LMATR ( \LMATL_{1} ) = 0 $ and $ D = 0 $, because $  \LMATT ( \LMATL_{1} ) =  \LMATR ( \LMATL_{1} )  + D $. 
From  $ D = 0 $ we obtain the equality in the LMI inequality,
and from $  \LMATR ( \LMATL_{1} ) = 0 $, we obtain that $ \LMATL_{1} = \LMATL_{0} $, because $ \DET \left( \LMATY \;\LMATH \;\LMATY^{T} \right) \neq 0 $
and then $ \LMATY \;\LMATH \;\LMATY^{T} $ is s.p.d.
If $ \LMATX = \LMATX \;  \LMATH \;  \LMATY^{T} \; \bigl( \LMATY \; \LMATH \; \LMATY^{T} \; \bigr)^{-1} \LMATY $, multiply both sides on the right by $ \LMATH \; \LMATX^{T}$, and then the equality for the LMI is obtained.
\end{proof}
The following lemma, cf. \cite{CHIPMAN-2011} Lemma 2.4.1, may be interpreted as a LMI generalization of the Rayleigh quotient, \cite{DUDA} p. 117.
\begin{LEMMA}  \label{lemma:RAYLEIGH}
Let $ \DIMLM , \DIMLN, \DIMLP \in \NAT $.
Let
$ \LMATB \in \REALS^{ \DIMLM \times \DIMLM } $ be an arbitrary real s.p.d. matrix, and let $ \LMATG \in \REALS^{ \DIMLN \times \DIMLM } $
and $ \LMATA \in \REALS^{ \DIMLP \times \DIMLM } $,  be otherwise arbitrary real matrices such that
$ \DET ( \LMATA \LMATB \LMATA^{T} ) \neq 0 $.
Then 
\begin{displaymath}
\LMATG \; \LMATB^{-1} \; \LMATG^{T} \; \ge 
\LMATG \; \LMATA^{T} \; \bigl( \LMATA \; \LMATB \; \LMATA^{T} \; \bigr)^{-1} \LMATA \; \LMATG^{T}
\end{displaymath}
with equality if and only if there exists $ \LMATL_{0} \in \REALS^{ \DIMLN \times \DIMLP } $, such that $ \LMATG = \LMATL_{0} \; \LMATA \; \LMATB $,
if and only if
\begin{displaymath}
\LMATG = \LMATG \; \LMATA^{T} \; \bigl( \LMATA \; \LMATB \; \LMATA^{T} \; \bigr)^{-1} \LMATA \; \LMATB \; 
\end{displaymath}
\end{LEMMA}
\begin{proof}
Since $ \LMATB $ is s.p.d. then it has a unique s.p.d. square root $ \LMATB^{1/2} $, \cite{HORN} p. 405, and
we have $ \DET ( \LMATB ) \neq 0 $ and $ \DET ( \LMATB^{1/2 } ) \neq 0 $.
The result follows from the previous Lemma \ref{lemma:CAUCHY-SCHWARZ} taking, 
$ \LMATX = \LMATG \; \LMATB^{-1/2}$, $ \LMATY = \LMATA \; \LMATB^{1/2}$,
and $ H $ as the identity matrix of dimensions $ \DIMLM \times \DIMLM $. 
\end{proof}

\bibliographystyle{amsplain}

\end{document}